\documentclass{article}

\usepackage{microtype}
\usepackage{graphicx}
\usepackage{subfig}
\usepackage{booktabs} 
\usepackage{bbm}
\usepackage{mathtools}
\usepackage{amssymb}
\usepackage{amsmath}
\usepackage{amsthm}
\usepackage{xcolor}
\usepackage[ruled, linesnumbered]{algorithm2e}

\usepackage{hyperref}

\usepackage[accepted]{icml2023}

\usepackage[capitalize,noabbrev]{cleveref}

\theoremstyle{plain}
\newtheorem{theorem}{Theorem}[section]
\newtheorem{proposition}[theorem]{Proposition}
\newtheorem{lemma}[theorem]{Lemma}

\theoremstyle{definition}
\newtheorem{definition}[theorem]{Definition}

\theoremstyle{remark}

\usepackage[textsize=tiny]{todonotes}

\def\c{cr}
\def\hc{\hat{cr}}

\def\ind{\mathbbm{1}}

\def\E{\mathbb{E}}
\def\R{\mathbb{R}}

\def\dc{Data\_Copy\_Detect}
\def\a{\kappa}
\def\hq{\widehat{q(B(x, r))}}
\def\alpaca{\epsilon}

\icmltitlerunning{Data-copying}

\begin{document}

\twocolumn[
\icmltitle{Data-Copying in Generative Models: A Formal Framework}



\icmlsetsymbol{equal}{*}

\begin{icmlauthorlist}
\icmlauthor{Robi Bhattacharjee}{yyy}
\icmlauthor{Sanjoy Dasgupta}{yyy}
\icmlauthor{Kamalika Chaudhuri}{yyy}
\end{icmlauthorlist}

\icmlaffiliation{yyy}{University of California, San Diego}

\icmlcorrespondingauthor{Robi Bhattacharjee}{rcbhatta@eng.ucsd.edu}

\icmlkeywords{Machine Learning, ICML}

\vskip 0.3in
]



\printAffiliationsAndNotice{\icmlEqualContribution} 

\begin{abstract}
There has been some recent interest in detecting and addressing memorization of training data by deep neural networks. A formal framework for memorization in generative models, called ``data-copying'' was proposed by Meehan et. al (2020). We build upon their work to show that their framework may fail to detect certain kinds of blatant memorization. Motivated by this and the theory of non-parametric methods, we provide an alternative definition of data-copying that applies more locally. We provide a method to detect data-copying, and provably show that it works with high probability when enough data is available. We also provide lower bounds that characterize the sample requirement for reliable detection.
\end{abstract}

\section{Introduction}

Deep generative models have shown impressive performance. However, given how large, diverse, and uncurated their training sets are, a big question is whether, how often, and how closely they are memorizing their training data. This question has been of considerable interest in generative modeling~\citep{lopez2016revisiting,XHYGSWK18} as well as supervised learning~\citep{BBFST21, Feldman20}. However, a clean and formal definition of memorization that captures the numerous complex aspects of the problem, particularly in the context of continuous data such as images, has largely been elusive.

For generative models,~\cite{MCD2020} proposed a formal definition of memorization called ``data-copying'', and showed that it was orthogonal to various prior notions of overfitting such as mode collapse~\citep{TT20}, mode dropping~\citep{YFWYC20}, and precision-recall~\citep{SBLBG18}. Specifically, their definition looks at three datasets -- a training set, a set of generated example, and an independent test set. Data-copying happens when the training points are considerably closer on average to the generated data points than to an independently drawn test sample. Otherwise, if the training points are further on average to the generated points than test, then there is underfitting. They proposed a three sample test to detect this kind of data-copying, and empirically showed that their test had good performance.

\begin{figure}[ht]
	\includegraphics[width=.45\textwidth]{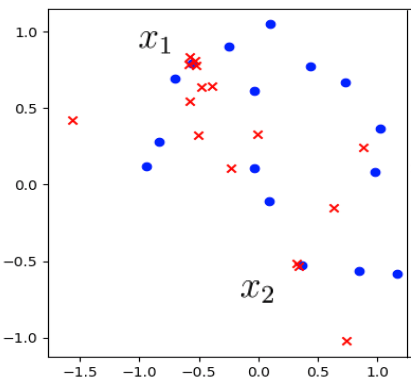}
	\caption{In this figure, the blue points are sampled from the halfmoons dataset (with Gaussian noise). The red points are sampled from a generated distribution that is a mixture of (40 \%) blatant data copier (that outputs a random subset of the training set), and (60 \%) a noisy underfit version of halfmoons. Although the generated distribution is clearly doing some form of copying at points $x_1$ and $x_2$, detecting this is challenging because of the canceling effect of the underfit points.}
	
	\label{fig:page_2_figure}
\end{figure}


However, despite its practical success, this method may not capture even blatant cases of memorization. To see this, consider the example illustrated in Figure \ref{fig:page_2_figure}, in which a generated model for the halfmoons dataset outputs one of its training points with probability $0.4$, and otherwise outputs a random point from an underfit distribution. When the test of~\cite{MCD2020} is applied to this distribution, it is unable to detect any form of data copying; the generated samples drawn from the underfit distribution are sufficient to cancel out the effect of the memorized examples. Nevertheless, this generative model is clearly an egregious memorizer as shown in points $x_1$ and $x_2$ of Figure \ref{fig:page_2_figure}.

This example suggests a notion of \textit{point-wise} data copying, where a model $q$ can be thought of as copying a given training point $x$. Such a notion would be able to detect $q$'s behavior nearby $x_1$ and $x_2$ regardless of the confounding samples that appear at a global level. This stands in contrast to the more global distance based approach taken in Meehan et. al. which is unable to detect such instances. Motivated by this, we propose an alternative point-by-point approach to defining data-copying.

We say that a generative model $q$  data-copies an individual training point, $x$, if it has an unusually high concentration in a small area centered at $x$. Intuitively, this implies $q$ is highly likely to output examples that are very similar to $x$. In the example above, this definition would flag $q$ as copying $x_1$ and $x_2$. 

To parlay this definition into a global measure of data-copying, we define the overall \textit{data-copying rate} as the total fraction of examples from $q$ that are copied from some training example. In the example above, this rate is $40\%$, as this is the fraction of examples that are blatant copies of the training data.

\begin{figure*}[ht]
    \subfloat{\includegraphics[width=.3\textwidth]{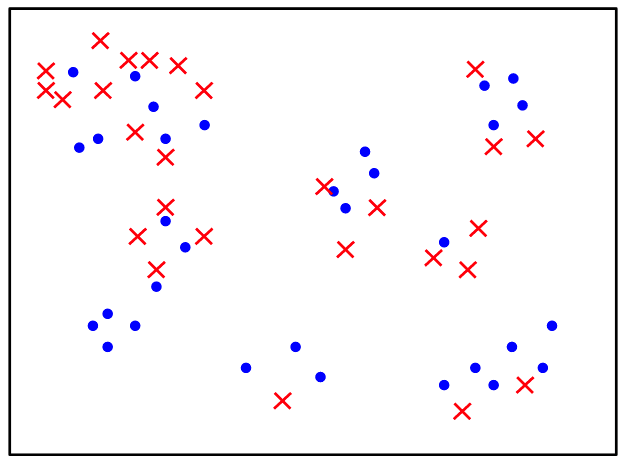}}\hfill
	\subfloat{\includegraphics[width=.3\textwidth]{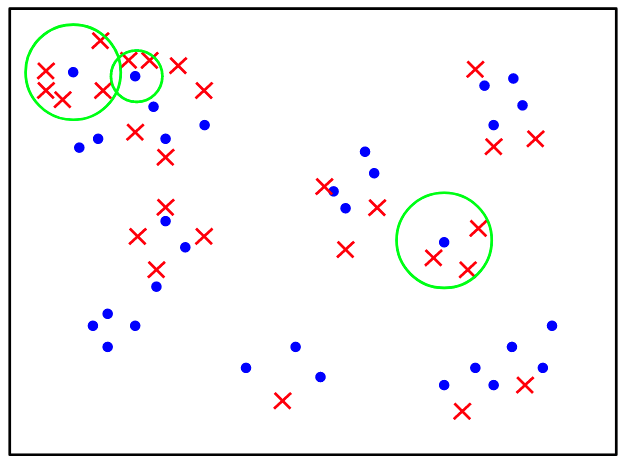}}\hfill
	\subfloat{\includegraphics[width=.3\textwidth]{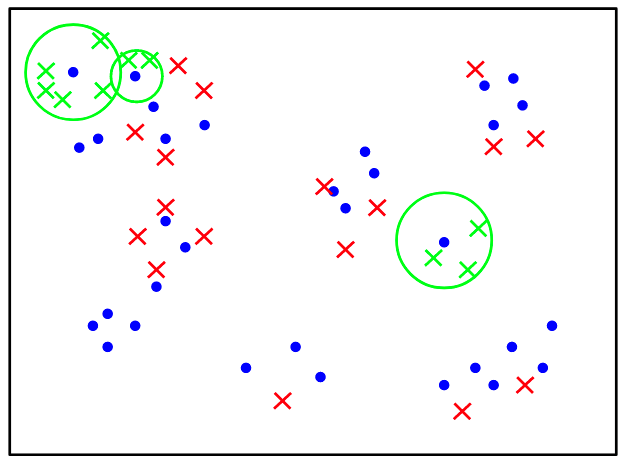}}
	\caption{In the three panels above, the blue points are a training sample from $p$, and the red points are generated examples from $q$. In the middle panel, we highlight in green regions that are defined to be \textit{data-copying regions}, as $q$ overrepresents them with comparison to $p$. In the third panel, we then color all points from $q$ that are considered to be copied green.}
	
	\label{fig:triptic}
\end{figure*}


Next, we consider how to detect data-copying according to this definition. To this end, we provide an algorithm, \dc{}, that outputs an estimate for the overall data-copying rate. We then show that under a natural smoothness assumption on the data distribution, which we call \textit{regularity}, \dc{} is able to guarantee an accurate estimate of the total data-copying rate. We then give an upper bound on the amount of data needed for doing so. 

We complement our algorithm with a lower bound on the minimum amount of a data needed for data-copying detection. Our lower bound also implies that some sort of smoothness condition (such as regularity) is necessary for guaranteed data-copying detection; otherwise, the required amount of data can be driven arbitrarily high.

\subsection{Related Work}

Recently, understanding failure modes for generative models has been an important growing body of work e.g. \citep{SGZCRC16, RW18, SBLBG18}. However, much of this work has been focused on other forms of overfitting, such as mode dropping or mode collapse.

A more related notion of overfitting is \textit{memorization} \citep{lopez2016revisiting,XHYGSWK18, C18}, in which a model outputs exact copies of its training data. This has been studied in both supervised \citep{BBFST21, Feldman20} and unsupervised \citep{BGWC21, CHCW21} contexts. Memorization has also been considered in language generation models \cite{Carlini22}. 

The first work to explicitly consider the more general notion of \textit{data-copying} is \citep{MCD2020}, which gives a three sample test for data-copy detection. We include an empirical comparison between our methods in Section \ref{sec:experiments}, where we demonstrate that ours is able to capture certain forms of data-copying that theirs is not. 

Finally, we note that this work focuses on detecting natural forms of memorization or data-copying, that likely arises out of poor generalization, and is not concerned with detecting \textit{adversarial} memorization or prompting, such as in \cite{Carlini19}, that are designed to obtain sensitive information about the training set. This is reflected in our definition and detection algorithm which look at the specific generative model, and not the algorithm that trains it.  Perhaps the best approach to prevent adversarial  memorization is training the model with differential privacy~\cite{Dwork06}, which ensures that the model does not change much when one training sample changes. However such solutions come at an utility cost. 

\section{A Formal Definition of Data-Copying}

We begin with the following question: what does it mean for a generated distribution $q$ to copy a single training example $x$? Intuitively, this means that $q$ is guilty of overfitting $x$ in some way, and consequently produces examples that are very similar to it. 

However, determining what constitutes a `very similar'  generated example must be done contextually. Otherwise the original data distribution, $p$, may itself be considered a copier, as it will output points nearby $x$ with some frequency depending on its density at $x$. Thus, we posit that $q$ data copies training point $x$ if it has a significantly higher concentration nearby $x$ than $p$ does. We express this in the following definition. 

\begin{definition}\label{defn:data_copy}
Let $p$ be a data distribution, $S \sim p^n$ a training sample, and $q$ be a generated distribution trained on $S$. Let $x \in S$ be a training  point, and let $\lambda > 1$ and $0 < \gamma < 1$ be constants. A generated example $x' \sim q$ is said to be a \textbf{$(\lambda, \gamma)$-copy} of $x$ if there exists a ball $B$ centered at $x$ (i.e. $\{x': ||x' - x|| \leq r\}$) such that following hold:
\begin{itemize}
	\item $x' \in B$.
	\item $q(B) \geq \lambda p(B)$
	\item $p(B) \leq \gamma$
\end{itemize}
\end{definition}

Here $q(B)$ and $p(B)$ denote the probability mass assigned to $B$ by $p$ and $q$ respectively.

The parameters $\lambda$ and $\gamma$ are user chosen parameters that characterize data-copying. $\lambda$ represents the rate at which $q$ must overrepresent points close to $x$, with higher values of $\lambda$ corresponding to more egregious examples of data-copying. $\gamma$ represents the maximum size (by probability mass) of a region that is considered to be data-copying -- the ball $B$ represents all points that are ``copies" of $x$. Together, $\lambda$ and $\gamma$ serve as practitioner controlled knobs that characterize data-copying about $x$.

Our definition is illustrated in Figure \ref{fig:triptic} -- the training data is shown in blue, and generated samples are shown in red. For each training point, we highlight a region (in green) about that point in which the red density is much higher than the blue density, thus constituting data-copying. The intuition for this is that the red points within any ball can be thought of as ``copies" of the blue point centered in the ball.

Having defined data-copying with respect to a single training example, we can naturally extend this notion to the entire training dataset. We say that $x' \sim q$ is copied from training set $S$ if $x'$ is a $(\lambda,\gamma)$-copy of some training example $x \in S$. We then define the \textit{data-copy rate} of $q$ as the fraction of examples it generates that are copied from $S$. Formally, we have the following: 

\begin{definition}
Let $p, S, q, \lambda,$ and $\gamma$ be as defined in Definition \ref{defn:data_copy}. Then the \textbf{data-copy rate}, $\c\left(q, \lambda, \gamma\right)$ of $q$ (with respect to $p, S$) is the fraction of examples from $q$ that are $(\lambda, \gamma)$-copied. That is, $$\c\left(q, \lambda, \gamma\right) = \Pr_{x' \sim q}[q\text{ }(\lambda,\gamma)\text{-copies }x'].$$ In cases where $\lambda, \gamma$ are fixed, we use $\c_q = \c(q, \lambda, \gamma)$ to denote the data-copy rate.
\end{definition}

Despite its seeming global nature, $\c_q$ is simply an aggregation of the point by point data-copying done by $q$ over its entire training set. As we will later see, estimating $\c_q$ is often reduced to determining which subset of the training data $q$ copies. 

\subsection{Examples of data-copying}

We now give several examples illustrating our definitions. In all cases, we let $p$ be a data distribution, $S$, a training sample from $p$, and $q$, a generated distribution that is trained over $S$. 

\paragraph{The uniform distribution over $S$:} In this example, $q$ is an egregious data copier that memorizes its training set and randomly outputs a training point. This can be considered as the canonical \textit{worst} data copier. This is reflected in the value of $\c_q$ -- if $p$ is a continuous distribution with finite probability density, then for any $x \in S$, there exists a ball $B$ centered at $x$ for which $q(B) >> p(B)$. It follows that $q$ $(\lambda,\gamma)$- copies $x$ for all $x \in S$ which implies that $\c_q = 1$.

\paragraph{The perfect generative model: $q = p$:} In this case, $q(B) = p(B)$ for all balls, $B$, which implies that $q$ does not perform any data-copying (Definition \ref{defn:data_copy}). It follows that $\c_q = 0$, matching the intuition that $q$ does not data-copy at all.

\paragraph{Kernel Density Estimators:} Finally, we consider a more general situation, where $q$ is trained by a \textit{kernel density estimator} (KDE) over $S \sim p^n$. Recall that a kernel density estimator outputs a generated distribution, $q$, with pdf defined by $$q(x) = \frac{1}{n\sigma_n}\sum_{x_i \in S} K\left(\frac{x - x_i}{\sigma_n}\right).$$ Here, $K$ is a kernel similarity function, and $\sigma_n$ is the bandwidth parameter. It is known that for $\sigma_n = O(n^{-1/5})$, $q$ converges towards $p$ for sufficiently well behaved probability distributions. 

Despite this guarantee, KDEs intuitively appear to perform some form of data-copying -- after all they implicitly include each training point in memory as it forms a portion of their outputted pdf. However, recall that our main focus is in understanding \textit{overfitting} due to data-copying. That is, we view data-copying as a function of the outputted pdf, $q$, and not of the training algorithm used. 

To this end, for KDEs the question of data-copying reduces to the question of whether $q$ overrepresents areas around its training points. As one would expect, this occurs \textit{before} we reach the large sample limit. This is expressed in the following theorem.

\begin{theorem}\label{thm:KDE}
Let $1 < \lambda$ and $\gamma > 0$. Let $\sigma_n$ be a sequence of bandwidths and $K$ be any regular kernel function. For any $n > 0$ there exists a probability distribution $\pi$ with full support over $\R^d$ such that with probability at least $\frac{1}{3}$ over $S \sim \pi^n$, a KDE trained with bandwidth $\sigma_n$ and kernel function $K$ has data-copy rate $\c_q \geq \frac{1}{10}$.
\end{theorem}

This theorem completes the picture for KDEs with regards to data-copying -- when $n$ is too low, it is possible for the KDE to have a significant amount of data-copying, but as $n$ continues to grow, this is eventually smoothed out.

\paragraph{The Halfmoons dataset}

Returning to the example given in Figure \ref{fig:page_2_figure}, observe that our definition exactly captures the notion of data-copying that occurs at points $x_1$ and $x_2$. For even strict choices of $\lambda$ and $\gamma$, Definition \ref{defn:data_copy} indicates that the red distribution copies both $x_1$ and $x_2$. Furthermore, the data-copy rate, $\c_q$, is $40\%$ by construction, as this is the proportion of points that are outputted nearby $x_1$ and $x_2$.

\subsection{Limitations of our definition}\label{sec:limitations}

Definition \ref{defn:data_copy} implicitly assumes that the goal of the generator is to output a distribution $q$ that approaches $p$ in a mathematical sense; a perfect generator would output $q$ so that $q(M) = p(M)$ for all measurable sets. In particular, instances where $q$ outputs examples that are far away from the training data are considered completely irrelevant in our definition.

This restriction prevents our definition from capturing instances in which $q$ memorizes its training data and then applies some sort of transformation to it. For example, consider an image generator that applies a color filter to its training data. This would not be considered a data-copier as its output would be quite far from the training data in pixel space. Nevertheless, such a generated distribution can be very reasonably considered as an egregious data copier, and a cursory investigation between its training data and its outputs would reveal as much. 

The key difference in this example is that the generative algorithm is no longer trying to closely approximate $p$ with $q$ -- it is rather trying to do so in some kind of transformed space. Capturing such interactions is beyond the scope of our paper, and we firmly restrict ourselves to the case where a generator is evaluated based on how close $q$ is to $p$ with respect to their measures over the input space. 

\section{Detecting data-copying}

Having defined $\c_q$, we now turn our attention towards \textit{estimating it.} To formalize this problem, we will require a few definitions. We begin by defining a generative algorithm.

\begin{definition}
A \textbf{generative algorithm}, $A$, is a potentially randomized algorithm that outputs a distribution $q$ over $\R^d$ given an input of training points, $S \subset \R^d$. We denote this relationship by $q \sim A(S)$.
\end{definition}

This paradigm captures most typical generative algorithms including both non-parametric methods such as KDEs and parametric methods such as variational autoencoders.

As an important distinction, in this work we define data-copying as a property of the generated distribution, $q$, rather than the generative algorithm, $A$. This is reflected in our definition which is given solely with respect to $q, S,$ and $p$. For the purposes of this paper, $A$ can be considered an arbitrary process that takes $S$ and outputs a distribution $q$. We include it in our definitions to emphasize that while $S$ is an i.i.d sample from $p$, it is \textit{not} independent from $q$. 

Next, we define a \textit{data-copying detector} as an algorithm that estimates $\c_q$ based on access to the training sample, $S$, along with the ability to draw any number of samples from $q$. The latter assumption is quite typical as sampling from $q$ is a purely computational operation. We do not assume any access to $p$ beyond the training sample $S$. Formally, we have the following definition.

\begin{definition}\label{def:data_copy_detector}
A \textbf{data-copying detector} is an algorithm $D$ that takes as input a training sample, $S \sim p^n$, and access to a sampling oracle for $q \sim A(S)$ (where $A$ is an arbitrary generative algorithm). $D$ then outputs an estimate, $D(S, q) = \hc_q$, for the data-copy rate of $q$. 
\end{definition}

Naturally, we assume $D$ has access to $\lambda, \gamma >0$ (as these are practitioner chosen values), and by convention don't include $\lambda, \gamma$ as formal inputs into $D$. 

The goal of a data-copying detector is to provide accurate estimates for $\c_q$. However, the precise definition of $\c_q$ poses an issue: data-copy rates for varying values of $\lambda$ and $ \gamma$ can vastly differ. This is because $\lambda, \gamma$ act as thresholds with everything above the threshold being counted, and everything below it being discarded. Since $\lambda, \gamma$ cannot be perfectly accounted for, we will require some tolerance in dealing with them. This motivates the following.

\begin{definition}\label{defn:approx_data_copy_rate}
Let $0 < \alpaca$ be a tolerance parameter. Then the \textbf{approximate data-copy rates}, $\c_q^{-\alpaca}$ and $\c_q^\alpaca$, are defined as the values of $\c_q$ when the parameters $(\lambda, \gamma)$ are shifted by a factor of $(1+\alpaca)$ to respectively decrease and increase the copy rate. That is, $$\c_q^{-\alpaca} = \c\left(q, \lambda (1+\alpaca), \gamma (1+\alpaca)^{-1}\right),$$ $$\c_q^{\alpaca} = \c\left(q, \lambda (1+\alpaca)^{-1}, \gamma (1+\alpaca)\right).$$
\end{definition}

The shifts in $\lambda$ and $\gamma$ are chosen as above because increasing $\lambda$ and decreasing $\gamma$ both reduce $\c_q$ seeing as both result in more restrictive conditions for what qualifies as data-copying. Conversely, decreasing $\lambda$ and increasing $\gamma$ has the opposite effect. It follows that $$\c_q^{-\alpaca} \leq \c_q \leq \c_q^{\alpaca},$$ meaning that $\c_q^{-\alpaca}$ and $\c_q^{\alpaca}$ are lower and upper bounds on $\c_q$. 

In the context of data-copying detection, the goal is now to estimate $\c_q$ in comparison to $\c_q^{\pm \alpaca}$. We formalize this by defining \textit{sample complexity} of a data-copying detector as the amount of data needed for accurate estimation of $\c_q$. 

\begin{definition}\label{def:sample_complexity}
Let $D$ be a data-copying detector and $p$ be a data distribution. Let $\epsilon, \delta > 0$ be standard tolerance parameters. Then $D$ has \textbf{sample complexity}, $m_p(\epsilon, \delta)$, with respect to $p$ if for all $n \geq m_p(\epsilon, \delta)$, $\lambda >1$, $0 < \gamma < 1$, and generative algorithms $A$, with probability at least $1 - \delta$ over $S \sim p^n$ and $q \sim A(S)$, $$\c_q^{-\alpaca} - \epsilon \leq D(S, q) \leq \c_q^{\alpaca} + \epsilon.$$
\end{definition}

Here the parameter $\epsilon$ takes on a somewhat expanded as it is both used to additively bound our estimation of $\c_q$ and to multiplicatively bound $\lambda$ and $\gamma$.

Observe that there is no mention of the number of calls that $D$ makes to its sampling oracle for $q$. This is because samples from $q$ are viewed as \textit{purely computational}, as they don't require any natural data source. In most cases, $q$ is simply some type of generative model (such as a VAE or a GAN), and thus sampling from $q$ is a matter of running the corresponding neural network.

\section{Regular Distributions}\label{sec:regular_dist}

Our definition of data-copying (Definition \ref{defn:data_copy}) motivates a straightforward point by point method for data-copying detection, in which for every training point, $x_i$, we compute the largest ball $B_i$ centered at $x_i$ for which $q(B_i) \geq \lambda p(B_i)$ and $p(B_i) \leq \gamma$. Assuming we compute these balls accurately, we can then query samples from $q$ to estimate the total rate at which $q$ outputs within those balls, giving us our estimate of $\c_q$.

The key ingredient necessary for this idea to work is to be able to reliably estimate the masses, $q(B)$ and $p(B)$ for any ball in $\R^d$. The standard approach to doing this is through \textit{uniform convergence}, in which large samples of points are drawn from $p$ and $q$ (in $p$'s case we use $S$), and then the mass of a ball is estimated by counting the proportion of sampled points within it. For balls with a sufficient number of points (typically $O( d\log n)$), standard uniform convergence arguments show that these estimates are reliable.

However, this method has a major pitfall for our purpose -- in most cases the balls $B_i$ will be very small because data-copying intrinsically deals with points that are very close to a given training point. While one might hope that we can simply ignore all balls below a certain threshold, this does not work either, as the sheer number of balls being considered means that their union could be highly non-trivial. 

To circumvent this issue, we will introduce an interpolation technique that estimates the probability mass of a small ball by scaling down the mass of a sufficiently large ball with the same center. While obtaining a general guarantee is impossible -- there exist pathological  distributions that drastically change their behavior at small scales -- it turns out there is a relatively natural condition under which such interpolation will work. We refer to this condition as \textit{regularity,} which is defined as follows.

\begin{definition}\label{def:regular}
Let $k> 0$ be an integer. A probability distribution $p$ is \textbf{$k$-regular} the following holds. For all $\alpaca > 0$, there exists a constant $0 < p_\alpaca \leq 1$ such that for all $x$ in the support of $p$, if $0 < s < r$ satisfies that $p(B(x, r)) \leq p_\alpaca$, then $$\left(1+\frac{\alpaca}{3}\right)^{-1}\frac{r^k}{s^{k}} \leq \frac{p(B(x, r))}{p(B(x, s))} \leq \left(1+\frac{\alpaca}{3}\right)\frac{r^k}{s^{k}}.$$ Finally, a distribution is \textbf{regular} if it is $k$-regular for some integer $k > 0$. 
\end{definition}

Here we let $B(x, r) = \{x': ||x - x'|| \leq r\}$ denote the closed $\ell_2$ ball centered at $x$ with radius $r$. 

The main intuition for a $k$-regular distribution is that at a sufficiently small scale, its probability mass scales with distance according to a power law, determined by $k$. The parameter $k$ dictates how the probability density behaves with respect to the distance scale. In most common examples, $k$ will equal the \textit{intrinsic dimension}  of $p$.

As a technical note, we use an error factor of $\frac{\alpaca}{3}$ instead of $\alpaca$ for technical details that enable cleaner statements and proofs in our results (presented later). 

\subsection{Distributions with Manifold Support}

We now give an important class of $k$-regular distributions.

\begin{proposition}\label{prop:manifold_works}
Let $p$ be a probability distribution with support precisely equal to a compact $k$ dimensional sub-manifold (with or without boundary) of $\R^d$, $M$. Additionally, suppose that $p$ has a continuous density function over $M$. Then it follows that $p$ is $k$-regular.
\end{proposition}

Proposition \ref{prop:manifold_works} implies that most data distributions that adhere to some sort of manifold-hypothesis will also exhibit regularity, with the regularity constant, $k$, being the intrinsic dimension of the manifold.

\subsection{Estimation over regular distributions}

We now turn our attention towards designing estimation algorithms over regular distributions, with our main goal being to estimate the probability mass of arbitrarily small balls. We begin by first addressing a slight technical detail -- although the data distribution $p$ may be regular, this does not necessarily mean that the regularity constant, $k$, is known. Knowledge of $k$ is crucial because it determines how to properly interpolate probability masses from large radius balls to smaller ones. 

Luckily, estimating $k$ turns out to be an extremely well studied task, as for most probability distributions, $k$ is a measure of the \textit{intrinsic dimension}. Because there is a wide body of literature in this topic, we will assume from this point that $k$ has been correctly estimated from $S$ using any known algorithm for doing so (for example \cite{BJPR22}). Nevertheless, for completeness, we provide an algorithm with provable guarantees for estimating $k$ (along with a corresponding bound on the amount of needed data) in Appendix \ref{sec:estimating_alpha}.

We now return to the problem of $p(B(x, r))$ for a small value of $r$, and present an algorithm, $Est(x, r, S)$ (Algorithm \ref{alg:estimate}), that estimates $p(B(x, r))$ from an i.i.d sample $S \sim p^n$.

\begin{algorithm}
   \caption{$Est(x, r, S)$}
   \label{alg:estimate}

   \DontPrintSemicolon
   
	$n \leftarrow |S|$\;
	
   $b \leftarrow O\left(\frac{d \ln \frac{n}{\delta}}{\epsilon^2} \right)$\;
   
   $r_* = \min \{s > 0, |S \cap B(x, s)| = b\}$.\;
   
   \uIf{$r_* > r$}{
   Return $\frac{br^k}{nr_*^k}$\;
   }
   \uElse {
	Return $\frac{|T \cap B(x, r)|}{n}$\;
	}

\end{algorithm}

$Est$ uses two ideas: first, it leverages standard uniform convergence results to estimate the probability mass of all balls that contain a sufficient number of training examples $(k)$. Second, it estimates the mass of smaller balls by interpolating from its estimates from larger balls. The $k$-regularity assumption is crucial for this second step as it is the basis on which such interpolation is done. 

$Est$ has the following performance guarantee, which follows from standard uniform convergence bounds and the definition of $k$-regularity. 
\begin{proposition}\label{prop:est_works}
Let $p$ be a regular distribution, and let $\alpaca >0$ be arbitrary. Then if $n = O\left(\frac{d\ln\frac{d}{\delta \alpaca p_\alpaca}}{\alpaca^2 p_\alpaca}\right)$ with probability at least $1 - \delta$ over $S \sim p^n$, for all $x \in \R^d$ and $r > 0$, $$\left(1+\frac{\alpaca}{2}\right)^{-1}\leq \frac{Est(x, r, S)}{p(B(x, r))} \leq \left(1+\frac{\alpaca}{2}\right).$$
\end{proposition}

\section{A Data-copy detecting algorithm}

\begin{algorithm}    

\caption{$DataCopyDetect(S, q, m)$}
\label{alg:main}   

   \DontPrintSemicolon
   
   $m \leftarrow O\left(\frac{dn^2\ln \frac{nd}{\delta\epsilon}}{\epsilon^4}\right)$\;
   
   Sample $T \sim q^m$\;
   
   $\{x_1, x_2, \dots, x_n\} \leftarrow S$\;
   
   $\{z_1, z_2, \dots, z_m\} \leftarrow T$\;

	\For{$i = 1, \dots, n$}{
	
	Let $p_i(r)$ denote $Est(x_i, r, S)$\;
	
	Let $q_i(r)$ denote $\frac{|B(x_i, r) \cap T|}{m}$\;
	
	$radii \leftarrow \{||z - x_i||: z \in T\} \cup \{0\}$\;
	
	$radii \leftarrow \{r: p_i(r) \leq \gamma, r \in radii\}$\;

	$r_i^* \leftarrow \max \{r: q_i(r) \geq \lambda p_i(r), r \in radii\}$\;
		
	}
	Sample $U \sim q^{20/\epsilon^2}$\;
	$V \leftarrow U \cap \left(\bigcup_{i=1}^n B(x_i, r_i^*)\right)$\;
	Return $\frac{|V|}{|U|}$.\;

\end{algorithm}

We now now leverage our subroutine, $Est$, to construct a data-copying detector, $Data\_Copy\_Detect$ (Algorithm \ref{alg:main}), that has bounded sample complexity when $p$ is a regular distribution. Like all data-copying detectors (Definition \ref{def:data_copy_detector}), $Data\_Copy\_Detect$ takes as input the training sample $S$, along with the ability to sample from a generated distribution $q$ that is trained from $S$. It then performs the following steps:
\begin{enumerate}
	\item (line 1) Draw an i.i.d sample of $m = O\left(\frac{dn^2\ln \frac{nd}{\delta\epsilon}}{\epsilon^4}\right)$ points from $q$. 
	\item (lines 6 - 10) For each training point, $x_i$, determine the largest radius $r_i$ for which 
	\begin{equation*}
	\begin{split}
	&\frac{|B(x_i, r_i) \cap T|}{m} \geq \lambda Est(x_i, r_i ,S), \\ 
	&Est(x_i, r_i , S) \leq \gamma.
	\end{split}
	\end{equation*}
	\item (lines 12 - 13) Draw a fresh sample of points from $U \sim q^{O(1/\epsilon^2)}$, and use it to estimate the probability mass under $q$ of $\cup_{i=1}^n B(x_i, r_i)$.
\end{enumerate}

In the first step, we draw a \textit{large} sample from $q$. While this is considerably larger than the amount of training data we have, we note that samples from $q$ are considered free, and thus do not affect the sample complexity. The reason we need this many samples is simple -- unlike $p$, $q$ is not necessarily regular, and consequently we need enough points to properly estimate $q$ around every training point in $S$.

The core technical details of $\dc{}$ are contained within step 2, in which data-copying regions surrounding each training point, $x_i$, are found. We use $Est(x, r, S)$ and $\frac{|B(x, r) \cap T|}{m}$ as proxies for $p$ and $q$ in Definition \ref{defn:data_copy}, and then search for the maximal radius $r_i$ over which the desired criteria of data-copying are met for these proxies.  

The only difficulty in doing this is that this could potentially require checking an infinite number of radii, $r_i$. Fortunately, this turns out not to be needed because of the following observation -- we only need to check radii at which a new point from $T$ is included in the estimation $q_i(r)$. This is because these our estimation for $q_i(r)$ does not change between them meaning that our estimate of the ratio between $q$ and $p$ is maximal nearby these points. 

Once we have computed $r_i$, all that is left is to estimate the data-copy rate by sampling $q$ once more to find the total mass of data-copying region, $\cup_{i=1}^n B(x_i, r_i)$. 

\subsection{Performance of Algorithm \ref{alg:main}}

We now show that given enough data, $\dc{}$ provides a close approximation of $\c_q$. 

\begin{theorem}\label{thm:upper_bound}
$\dc{}$ is a data-copying detector (Definition \ref{def:data_copy_detector}) with sample complexity at most $$m_p(\epsilon, \delta) = O\left(\frac{d\ln\frac{d}{\delta\alpaca p_\alpaca}}{\alpaca^2 p_\alpaca}\right),$$ for all regular distributions, $p$. 
\end{theorem}

Theorem \ref{alg:main} shows that our algorithm's sample complexity has standard relationships with the tolerance parameters, $\epsilon$ and $\delta$, along with the input space dimension $d$. However, it includes an additional factor of $\frac{1}{p_\epsilon}$, which is a distribution specific factor measuring the regularity of the probability distribution. Thus, our bound cannot be used to give a bound on the amount of data needed without having a bound on $p_\epsilon$. 

We consequently view our upper bound as more akin to a convergence result, as it implies that our algorithm is guaranteed to converge as the amount of data goes towards infinity.

\subsection{Applying Algorithm \ref{alg:main} to Halfmoons}\label{sec:experiments}

We now return to the example presented in Figure \ref{fig:halfmoons} and empirically investigate the following question: is our algorithm able to outperform the one given in \cite{MCD2020} over this example? 

To investigate this, we test both algorithms over a series of distributions by varying the parameter $\rho$, which is the proportion of points that are ``copied." Figure \ref{fig:halfmoons} demonstrates a case in which $\rho = 0.4$. Additionally, we include a parameter, $c$, for \cite{MCD2020}'s algorithm which represents the number of clusters the data is partitioned into (with $c$-means clustering) prior to running their test. Intuitively, a larger number of clusters means a better chance of detecting more localized data-copying.

The results are summarized in the following table where we indicate whether the algorithm determined a statistically significant amount of data-copying over the given generated distribution and corresponding training dataset. Full experimental details can be found in Sections \ref{sec:app_experiments} and \ref{sec:experiments_details} of the appendix. 
\begin{table}[h]
\caption{Statistical Significance of data-copying Rates over Halfmoons} \label{results_main}
\begin{center}
\begin{tabular}{ |c||c|c|c|c|c| } 
 \hline
 \textbf{Algo} & $\mathbf{q = p}$ & $\mathbf{\rho = 0.1}$ & $\mathbf{0.2}$ & $\mathbf{0.3}$ & $\mathbf{0.4}$ \\ 
 \hline
 \hline
 \textbf{Ours} & \color{blue}no & \color{red}yes & \color{red}yes & \color{red}yes & \color{red}yes \\ 
 \hline
 $\mathbf{c=1}$ & \color{blue}no & \color{blue}no & \color{blue}no & \color{blue}no & \color{blue}no \\ 
 \hline
 $\mathbf{c=5}$ & \color{blue}no & \color{blue}no & \color{blue}no & \color{blue}no & \color{red}yes \\ 
 \hline
 $\mathbf{c=10}$ & \color{blue}no & \color{blue}no & \color{blue}no & \color{blue}no & \color{red}yes \\ 
 \hline
 $\mathbf{c=20}$ & \color{blue}no & \color{blue}no& \color{blue}no & \color{red}yes & \color{red}yes\\ 
 \hline
\end{tabular}
\end{center}
\end{table}

As the table indicates, our algorithm is able to detect statistically significant data-copying rates in all cases it exists. By contrast, \cite{MCD2020}'s test is only capable of doing so when there is a large data-copy rate and when the number of clusters, $c$, is quite large.

\section{Is smoothness necessary for data copying detection?}\label{sec:lower_bound}

Algorithm \ref{alg:main}'s performance guarantee requires that the input distribution, $p$, be regular (Definition \ref{def:regular}). This condition is essential for the algorithm to successfully estimate the probability mass of arbitrarily small balls. Additionally, the parameter, $p_\epsilon$, plays a key role as it serves as a measure of how ``smooth" $p$ is with larger values implying a higher degree of smoothness. 

This motivates a natural question -- can data copying detection be done over unsmooth data distributions? Unfortunately, the answer turns out to be no. In the following result, we show that if the parameter, $p_\epsilon$ is allowed to be arbitrarily small, then this implies that for any data-copy detector, there exists $p$ for which the sample complexity is arbitrarily large.

\begin{theorem}\label{thm:lower_bound}
Let $B$ be a data-copying detector. Let $\epsilon = \delta = \frac{1}{3}$. Then, for all integers $\a > 0$, there exists a probability distribution $p$ such that $\frac{1}{9\a} \leq p_\alpaca \leq \frac{1}{\a}$, and $m_p(\epsilon, \delta) \geq \a$, implying that $$m_p(\epsilon, \delta) \geq \Omega\left(\frac{1}{p_\epsilon}\right).$$
\end{theorem}

Although Theorem \ref{thm:lower_bound} is restricted to regular distributions, it nevertheless demonstrates that a bound on smoothness is essential for data copying detection. In particular, non-regular distributions (with no bound on smoothness) can be thought of as a degenerate case in which $p_\epsilon = 0$. 

Additionally, Theorem \ref{thm:lower_bound} provides a lower bound that complements the Algorithm \ref{alg:main}'s performance guarantee (Theorem \ref{thm:upper_bound}). Both bounds have the same dependence on $p_\alpaca$ implying that our algorithm is optimal at least in regards to $p_\alpaca$. However, our upper bound is significantly larger in its dependence on $d$, the ambient dimension, and $\alpaca$, the tolerance parameter itself. 

While closing this gap remains an interesting direction for future work, we note that the existence of a gap isn't too surprising for our algorithm, $\dc{}$. This is because $\dc{}$ essentially relies on manually finding the entire region in which data-copying occurs, and doing this requires precise estimates of $p$ at all points in the training sample.  

Conversely, detecting data-copying only requires an \textit{overall} estimate for the data-copying rate, and doesn't necessarily require finding all of the corresponding regions. It is plausible that more sophisticated techniques might able to estimate the data-copy rate \textit{without} directly finding these regions.

\section{Conclusion}

In conclusion, we provide a new modified definition of ``data-copying'' or generating memorized training samples for generative models that addresses some of the failure modes of previous definitions~\cite{MCD2020}. We provide an algorithm for detecting data-copying according to our definition, establish performance guarantees, and show that at least some smoothness conditions are needed on the data distribution for successful detection. 

With regards to future work, one important direction is in addressing the limitations discussed in section \ref{sec:limitations}. Our definition and algorithm are centered around the assumption that the goal of a generative model is to output $q$ that is close to $p$ in a mathematical sense. As a result, we are unable to handle cases where the generator tries to generate \textit{transformed} examples that lie outside the support of the training distribution. For example, a generator restricted to outputting black and white images (when trained on color images) would remain completely undetected by our algorithm regardless of the degree with which it copies its training data. To this end, we are very interested in finding generalizations of our framework that are able to capture such broader forms of data-copying. 

\section*{Acknowledgments} 

We thank NSF under CNS 1804829 for research support.

\bibliography{references}
\bibliographystyle{icml2023}

\newpage
\appendix
\onecolumn

\section{An Example over the Halfmoons dataset}\label{sec:app_experiments}

In this section, we give an overview of our experiments over the Halfmoons dataset. Further details can be found in sec

\begin{figure*}[ht]
	\subfloat[$\rho = 0.1$]{\includegraphics[width=.45\textwidth]{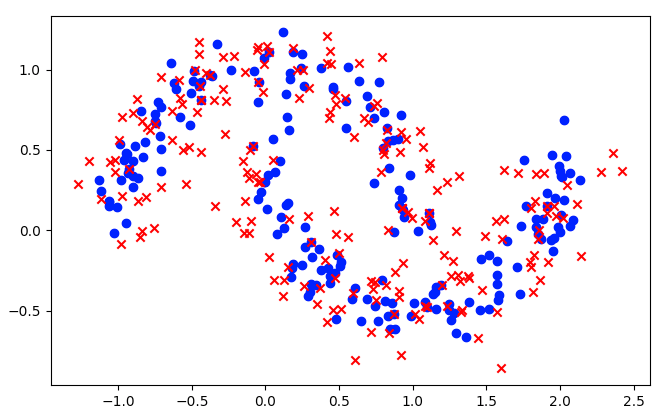}}\hfill
	\subfloat[$\rho = 0.4$]{\includegraphics[width=.45\textwidth]{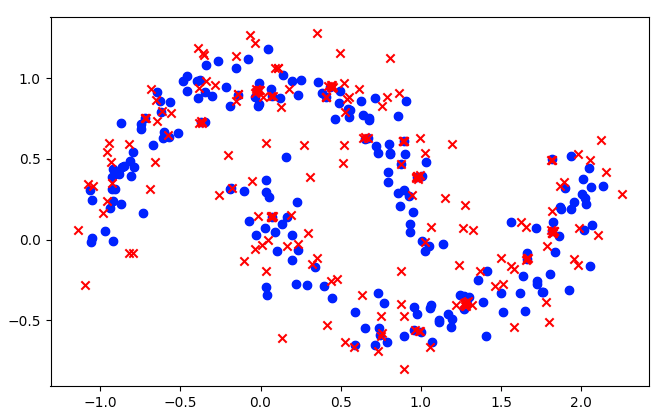}}
	\caption{In the two panels above, the blue points are a training sample from $p$, and the red points are generated examples from $q$. The parameter $\rho$ is the proportion of examples of $q$ that are generated by $q_{copy}$, with the rest of the examples being drawn from $q_{underfit}$. As $\rho$ increases, the rate of data-copying increases, which can be seen as the red points become increasingly clustered on top of a scattering of blue ones. However, due to $q_{underfit}$, there are still many red points that are relatively scattered from the blue points. At a global level, these effects average out making data-copying detection difficult for \cite{MCD2020}'s method.}
	
	\label{fig:halfmoons}
\end{figure*}

Our theoretical results show that given enough data, Algorithm \ref{alg:main} is guaranteed to detect data-copying. By contrast, the non-parametric test provided in \cite{MCD2020} can only guarantee detection in cases in which data-copying \textit{globally} occurs. For more local instances of data-copying, they rely on $k$-means clustering to partition the input space into localized regions, and then run their global test over each region separately.

Their approach clearly cannot detect all forms of data-copying -- a pathological generative distribution might copy in complex regions that are impossible to find using $k$-means clustering. However, for many practical examples considered in their paper, \cite{MCD2020} demonstrated considerable success with this approach. 

This motivates the following question: 
\begin{quote}
Do there exist natural data distributions over which Algorithm \ref{alg:main} offers a meaningful advantage?
\end{quote} 

We provide a partial answer to this question by experimentally comparing our approach with \cite{MCD2020}'s over a simple example on the half moons dataset.

\subsection{Experimental Setup}

\paragraph{Data Distribution:} Our data distribution, $p$, is the Halfmoon dataset with Gaussian noise ($\sigma = 0.1$).  

\paragraph{Generated Distribution:} Our generated distribution, $q$, is trained from an i.i.d sample of 2000 points from $p$, $S \sim p^{2000}$. Because our focus is on distinguishing \textit{data-copy detection} algorithms, we design $q$ to have a large amount of data-copying that is nevertheless subtle to detect. The key idea is to let $q$ be a mixture of two distributions, $q_{copy}$ and $q_{underfit}$. $q_{copy}$ will be an egregious data copier, and $q_{underfit}$ will be designed to average away the effects of $q_{copy}$. 

To construct $q_{copy}$, we first select a subset, $S' \subset S$, of $20$ training examples. Then, we define $q_{copy}$ to randomly output points from $S'$ combined with a small amount of spherical noise (with radius $0.02$). Thus, $q_{copy}$ can be sampled from by sampling a point, $x$, from $S'$ at uniform, and returning $x+\eta$ where $\eta$ is drawn at uniform from a disk of radius $0.02$.

To construct $q_{underfit}$, we combine our original data distribution, $p$, with a moderate amount of spherical noise (with radius $0.25$). Thus, $q_{underfit}$ can be sampled from by first sampling $x \sim p$, and returning $x+ \eta$ where $\eta$ is drawn at uniform from a disk of radius $0.25$. This distribution is meant to represent a fairly noisy and thus underfit version of $p$. 

Finally, we define $q$ as a mixture of $q_{copy}$ and $q_{underfit}$, with $q$ outputting a point from $q_{copy}$ with probability $\rho$. In total, we have $$q = \rho \cdot q_{copy} + (1-\rho) \cdot q_{good}.$$ We let, $\rho$, the weight of $q_{copy}$ within the mixture, be a varying parameter that gives rise to different generated distributions. Intuitively, the larger $\rho$ is, the higher the data-copying rate. This is illustrated in Figure \ref{fig:halfmoons}. In the both panels, we plot a sample of $200$ training points $p$ along with $200$ points from $q$. In the left panel, we let $\rho = 0.1$ in the right, we use $\rho = 0.4$. Although both cases show examples of data-copying, the right panel shows a visibly higher level of it. This is expected, as it is drawn from a distribution in which $q_{copy}$ is much more likely to be queried. 

\paragraph{Data-copying Detection Algorithms:} We run our algorithm, \dc{}, on $(S, q)$, We fix $\lambda = 20$ and $\gamma = 0.00025$ as constants for data-copy detection. $\lambda$ represents a healthy level of data-copying, and $\gamma = 0.00025$ ensures that our condition for 'copying' is quite stringent. Full details of our implementation (including our practical choices for parameters such as $b$ and $m$) are given in Appendix \ref{sec:experiments}.

For comparison, we also include an implementation of \cite{MCD2020}'s algorithm with varying amounts of clusters being used for the initial $k$-means clustering. To avoid confusion with the intrinsic dimension, $k$, we let $c$ denote the number of clusters, and consider $c \in \{1, 5, 10, 20\}$.

\subsection{Results}

The results are summarized in Table \ref{results}, with each column corresponding to a given choice of $p, q$ (determined by the parameter $\rho$), and each row corresponding to a separate data-copying detection algorithm. As a baseline, we include the case where $q = p$ (meaning we have a perfect generated distribution) in the first column.

We run our algorithm with parameters $\lambda$ and $\gamma$ fixed as $20$ and $0.00025$ in all cases. For \cite{MCD2020}'s algorithm, we consider their data-copy detection score over the most egregious cluster. 

Although our algorithm outputs real number estimates of the true data-copying rate, $\c_q$, \cite{MCD2020}'s algorithm outputs a score indicating the statistical significance of their metric under a null hypothesis of no data-copying occurring. To facilitate a simple comparison between our methods, for all algorithms, we simply output a simple yes or no to indicate whether our results were statistically significant up to the $p=0.05$ level. We include full results of our experiments along with several extensions (with varying parameters) in section \ref{sec:experiments_details}.

As expected, neither of our algorithms detect data-copying on the baseline, $q = p$. However, in all other cases, our algorithm successfully detects data-copying. On the other hand, for the smaller values of $\rho$, \cite{MCD2020}'s does not. Their algorithm is only able to achieve detection when the weight of $\rho = 0.4$, and even in this case they are unable to consistently do so.  

These results match the simple intuition of our algorithms. As seen in Figure \ref{fig:halfmoons}, the red data is sometimes very close to the blue data (when it comes from $q_{copy}$) but at other times fairly distant (when it comes from $q_{underfit}$). These effects have a strong canceling effect in \cite{MCD2020}'s test. However, our test is able to adjust for this by considering each training example separately.

\begin{table}[h]
\caption{Statistical Significance of data-copying Rates over Halfmoons} \label{results}
\begin{center}
\begin{tabular}{ |c||c|c|c|c|c| } 
 \hline
 \textbf{Algo} & $\mathbf{q = p}$ & $\mathbf{\rho = 0.1}$ & $\mathbf{0.2}$ & $\mathbf{0.3}$ & $\mathbf{0.4}$ \\ 
 \hline
 \hline
 \textbf{Ours} & \color{blue}no & \color{red}yes & \color{red}yes & \color{red}yes & \color{red}yes \\ 
 \hline
 $\mathbf{c=1}$ & \color{blue}no & \color{blue}no & \color{blue}no & \color{blue}no & \color{blue}no \\ 
 \hline
 $\mathbf{c=5}$ & \color{blue}no & \color{blue}no & \color{blue}no & \color{blue}no & \color{red}yes \\ 
 \hline
 $\mathbf{c=10}$ & \color{blue}no & \color{blue}no & \color{blue}no & \color{blue}no & \color{red}yes \\ 
 \hline
 $\mathbf{c=20}$ & \color{blue}no & \color{blue}no& \color{blue}no & \color{red}yes & \color{red}yes\\ 
 \hline
\end{tabular}
\end{center}
\end{table}

 \subsection{Further Experimental Details}\label{sec:experiments_details}

We begin by reviewing the definitions of $p$ and $q$. $p$ is the Halfmoons dataset with Gaussian noise $(\sigma = 0.1)$. To define $q$, we have a mixture of two distributions, $q_{copy}$ and $q_{underfit}$, which are defined as follows.

We draw $S \sim p^{2000}$ i.i.d, and then randomly select $S' \subset S$ with $|S'| = 20$. These points will form a basis for the support of $q_{copy}$. To sample $x \sim q_{copy}$, we take the following two steps.
\begin{enumerate}
	\item Sample $z \sim S'$ at uniform.
	\item Sample $\eta \sim U(B(0, 0.02))$, where $U(B(0, r))$ denotes the uniform distribution over the ball of radius $r$. 
	\item Output $x = z + \eta$.
\end{enumerate}
$q_{copy}$ can be thought of as an egregious data memorizer that injects a small amount of noise to give its inputs some (paltry) variety. 

By contrast, to sample $x \sim q_{underfit}$, we do the following:
\begin{enumerate}
	\item Sample $z \sim p$.
	\item Sample $\eta \sim U(B(0, 0.25))$.
	\item Output $x = z + \eta$.
\end{enumerate}
In this case, the larger amount of noise serves to induce \textit{underfitting}, in which $q_{copy}$ does not assign the support of $p$ enough probability mass. 

Finally, to sample from $q$, we do the following.
\begin{enumerate}
	\item With probability $\rho$, sample $x \sim q_{copy}$.
	\item With probability $1 - \rho$, sample $x \sim q_{underfit}$. 
\end{enumerate}

\paragraph{\cite{MCD2020}'s test:} Their test works as follows. Let $S$ denote the original training sample, $Q$ denote a sample of generated examples, with $Q \sim q^{n}$, and $P$ denote a fresh set of test examples, with $P \sim p^n$. They then check to see if $Q$ is systematically closer to $S$ than $P$, (thus suggesting data copying). To do so, they use a statistical test as follows:
\begin{enumerate}
	\item Let $S = \{x_1, x_2, \dots, x_n\}$, $P = \{y_1, y_2, \dots, y_n\}$, $Q = \{z_1, z_2, \dots, z_n\}$. 
	\item Let $\Delta$ denote the number of pairs $(i, j)$ for which $d(y_i, S) < d(z_j, S)$. A large value of $\Delta$ indicates that a \textit{small} amount of data copying, as it implies that $Q$ is further from $S$ than $P$. A small value of $\Delta$ indicates a \textbf{large} amount of data copying.
	\item Reflecting this, let $Z = \frac{\Delta - \frac{n^2}{2}}{\sqrt{\frac{n^2(2n+1)}{12}}}$. This gives a $Z$-score of $\Delta$. \cite{MCD2020} show that, $p = q$, then the probability of results as significant as $Z < -5$ would be at most the probability of getting a $-5\sigma$ event when sampling from a Gaussian. We use $Z < -3$ to indicate \textit{statistically significant results}, and output the corresponding $P$-values ($P =  0.0027$ being significant) in our results. 
\end{enumerate}

Finally, to account for data copying occurring within specific regions, \cite{MCD2020} perform a preprocessing step in which they cluster the training data, $S$ into $c$ regions using $k$-means clustering. They then run their test separately on each region by assigning points from $P$ and $Q$ into the regions containing them. We output the \textit{lowest} $Z$-score over any region, and vary the number of clusters with $c= 1, 5, 10, 20$. 

\paragraph{Our test:} We run Algorithm \ref{alg:main} with input $(S, q)$ with a few adjustments.
\begin{enumerate}
	\item We directly set $m = 200,000$. While the theoretical value of $m$ is significantly higher (growing $O(n^2)$), we note that this is primarily done for achieving theoretical guarantees. In practice, often a much lower amount of data is needed.
	\item For $Est(x, r, S)$, we set $b= 400$, which is a bit lower than the theoretically predicted value. As for $m$, we do this because for practical (and well-behaved) datasets, $Est(x, r, S)$ converges much more quickly than theory suggests. 
	\item We set $\lambda = 20$ and $\gamma = \frac{1}{4000}$, giving relatively stringent conditions on data copying. 
\end{enumerate}
Finally, our test outputs, $\hat{\c}_q$, which is an estimate of the data copy rate. Technically, any non-zero of $\hat{\c}_q$ indicates a degree of data copying. To facilitate a more direct comparison with \cite{MCD2020}, we convert our results into statistical tests by doing the following.
\begin{enumerate}
	\item We compute $\hat{\c}_p$, which is an estimate for the data copying rate when the generated distribution exactly equals $p$ over $1000$ different instances (each instance corresponding to a freshly drawn training set $S$).
	\item We then compute $\hat{\c}_q$ when $q$ is as above.
	\item We finally output the fraction of the time that $\hat{\c}_p > \hat{\c}_q$, thus giving us a P-value by giving us the rate at which the null-hypothesis gives results as significant as those that we observe. 
\end{enumerate} 

\paragraph{Results:} We give a more complete version of Table \ref{results}, with the $P$-values themselves being outputted in the table. For consistency, we output the median $P$-value obtained over 10 runs for each experiment. 

\begin{table}[h]
\caption{P-values of data-copying Rates over Halfmoons} \label{results:full}
\begin{center}
\begin{tabular}{ |c||c|c|c|c|c| } 
 \hline
 \textbf{Algo} & $\mathbf{q = p}$ & $\mathbf{\rho = 0.1}$ & $\mathbf{0.2}$ & $\mathbf{0.3}$ & $\mathbf{0.4}$ \\ 
 \hline
 \hline
 \textbf{Ours} & \color{blue}1.000 & \color{red}0.000 & \color{red}0.000 & \color{red}0.000 & \color{red}0.000 \\ 
 \hline
 $\mathbf{c=1}$ & \color{blue}0.5412 & \color{blue}1.000 & \color{blue}1.000 & \color{blue}0.858 & \color{blue}0.026 \\ 
 \hline
 $\mathbf{c=5}$ & \color{blue}0.113 & \color{blue}0.976 & \color{blue}0.780 & \color{blue}0.081 & \color{red}0.007 \\ 
 \hline
 $\mathbf{c=10}$ & \color{blue}0.090 & \color{blue}0.814 & \color{blue}0.294 & \color{blue}0.013 & \color{red}0.000 \\ 
 \hline
 $\mathbf{c=20}$ & \color{blue}0.035 & \color{blue}0.279& \color{blue}0.093 & \color{red}0.005 & \color{red}0.000\\ 
 \hline
\end{tabular}
\end{center}
\end{table}

\section{Estimating $k$}\label{sec:estimating_alpha}

The main idea of our method is to simply pick any point $x_i$ in the training sample, $S = \{x_1, x_2, \dots, x_n\}$, choose two small balls centered at $x_i$, and then measure the ratio of their probability masses as well as their radii. For sufficiently small balls, these ratios will be related by a power of $k$, and we can consequently just solve for an estimate of $k$, $\hat{k}$. Finally, since for our purposes it is extremely important that our estimate be \textit{exactly} correct, we round $\hat{k}$ to the nearest integer. While this clearly fails in cases that $k$ is not an integer, for most distributions $k$ precisely equals the dimension of the underlying data manifold (see for example Proposition \ref{prop:manifold_works}). These steps are enumerated in the following algorithm, $Estimate\_k(S)$. 

\begin{algorithm}
   \caption{$Estimate\_k(S)$}
   \label{alg:estimate_alpha}

   \DontPrintSemicolon

	$n \leftarrow |S|$\;
	
	Pick $x \in S$ arbitrarily.\;
	
   $b \leftarrow \frac{64(d+2)\ln\frac{16n}{\delta}}{\epsilon^2}.$\;
   
   $r_* = \min \{r: |S \cap B(x, r)| = 2b\}$.\;
   
  $s_* =\min \{s: |S \cap B(x, s)| = b\}$\;
  
  $\hat{k} = round \left(\frac{1}{\log_2 \frac{r_*}{s_*}} \right)$\;
  
  Return $\hat{k}$. 

\end{algorithm}

We now give sufficient conditions under which Algorithm \ref{alg:estimate_alpha} successfully recovers $k$. 

\begin{proposition}
Let $p$ be an $k$-regular distribution, and let $\delta > 0$ be arbitrary. Let $\phi = \frac{1}{2k}$. Then there exists a constant $C$ such that if $$n \geq C\frac{d\ln\frac{d}{\delta\phi p_\phi}}{\phi^2 p_\phi},$$ with probability at least $1 - \delta$ over $S \sim p^n$, $Estimate\_k(S) = k$. 
\end{proposition}

\begin{proof}
We begin by first applying standard uniform convergence over $\ell_2$ balls in $\R^d$ (which have a VC dimension of at most $d+2$). To this end, let $$\beta_n = \sqrt{\frac{4(d+2)\ln \frac{16n}{\delta}}{n}}.$$ Then by the standard result of Vapnik and Chervonenkis, with probability $1-\delta$ over $S \sim p^n$, for all $x \in \R^d$ and all $r > 0$, 
\begin{equation}\label{aeqn:VC}
\frac{|S \cap B(x,r)|}{n} - \beta_n\sqrt{\frac{|S \cap B(x,r)|}{n}} \leq p(B(x,r)) \leq \frac{|S \cap B(x,r)|}{n} + \beta_n^2 + \beta_n\sqrt{\frac{|S \cap B(x,r)|}{n}}.
\end{equation}

Next, assume that 
\begin{equation}\label{aeqn:bound_n}
n \geq \frac{1776(d+2)\ln \left(\frac{28416(d+2)}{\delta \phi^2 p_\phi} \right)}{\phi^2 p_\phi}.
\end{equation} 
It is clear that for an appropriate constant, we have $n = O \left(\frac{d\ln\frac{d}{\delta \phi p_\phi}}{\phi^2 p_\phi} \right)$. Thus, it suffices to show that if Equation \ref{aeqn:VC} holds, then $\hat{k} = k$ (as the former holds with probability $1-\delta$ over $S$). We now show the following claim.

\textbf{Claim:} Let $r > 0$ be any radius with $|S \cap B(x, r)| \geq b$. Then $$\left(1 + \frac{\phi}{9}\right)^{-1} \leq \frac{|S \cap B(x, r)|}{n p(B(x, r))} \leq \left(1 + \frac{\phi}{9} \right).$$ 

\begin{proof} From the definition of $b$, we have that \begin{equation}\label{aeqn:defining_b}\frac{b}{n} = \frac{400(d+2)\ln\frac{16n}{\delta}}{n\phi^2} = \frac{100\beta_n^2}{\phi^2}.\end{equation} Let $c = \sqrt{\frac{b'}{n\beta_n^2}}$. Then $b' \geq b$ implies that $c \geq \frac{10}{\phi}$. It follows that \begin{equation}\label{aeqn:c_stuff}\frac{c+1}{c^2} \leq \frac{1}{c-1} \leq \frac{\phi}{9}.\end{equation} Substituting Equations \ref{aeqn:defining_b} and \ref{aeqn:c_stuff} into  Equation \ref{aeqn:VC}, we have 
\begin{equation}\label{aeqn:epsilon_lower_bound}
\begin{split}
\frac{b'}{np(B(x, r))} &\geq \frac{\frac{b'}{n}}{\frac{b'}{n} + \beta_n^2 + \beta_n \sqrt{\frac{k'}{n}}} \\
&= \frac{c^2}{c^2 + 1 + c} \\
&= \left(1 + \frac{c+1}{c^2} \right)^{-1} \\
&\geq \left(1 + \frac{\phi}{9}\right)^{-1}
\end{split}
\end{equation}
and 
\begin{equation}\label{aeqn:epsilon_upper_bound}
\begin{split}
\begin{split}
\frac{b'}{np(B(x, r))} &\leq \frac{\frac{b'}{n}}{\frac{b'}{n} - \beta_n \sqrt{\frac{b'}{n}}} \\
&= \frac{c^2}{c^2 - c} \\
&= 1 + \frac{1}{c-1} \\
&\leq 1 + \frac{\phi}{9},
\end{split}
\end{split}
\end{equation}
Together, Equations \ref{aeqn:epsilon_lower_bound} and \ref{aeqn:epsilon_upper_bound} imply our claim.
\end{proof}

We now return to the proof of Proposition \ref{prop:est_works}. We now show that $p(B(x, s_*) \leq p(B(x, r_*)) \leq p_\phi$. To do so, we first bound $\beta_n^2$ as follows. We have, 
\begin{equation}\label{aeqn:bound_alpha_n}
\begin{split}
\beta_n^2 &= \frac{4(d+2)\ln(16n/\delta)}{n} \\
&= 4(d+2) \ln \left(\frac{28416(d+2)}{\delta \phi^2 p_\phi}\ln\left(\frac{28416(d+2)}{\delta \phi^2 p_\phi} \right)\right) \frac{\phi^2 p_\phi}{1776(d+2)\ln \left(\frac{28416(d+2)}{\delta \phi^2 p_\phi} \right)} \\
&\leq 8(d+2)\ln \left(\frac{28416(d+2)}{\delta \phi^2 p_\phi}\right) \frac{\phi^2 p_\phi}{1776(d+2)\ln \left(\frac{28416(d+2)}{\delta \phi^2 p_\phi} \right)} \\
&= \frac{p_\phi \phi^2}{222}.
\end{split}
\end{equation}
Next, by Equations \ref{aeqn:VC} and \ref{aeqn:bound_alpha_n} along with the fact that $b = \frac{100\beta_n^2}{\phi^2}$ (Equation \ref{aeqn:defining_b}) that 
\begin{equation*}
\begin{split}
p(B(x, r_*)) &\leq \frac{|S \cap B(x, r_*)|}{n} + \beta_n^2 + \beta_n \sqrt{\frac{|S \cap B(x, r_*)|}{n}} \\
&= \frac{2b}{n} + \beta_n^2 + \beta_n\sqrt{\frac{2b}{n}} \\
&= \beta_n^2\left(\frac{200}{\phi^2} +1 + \frac{20}{\phi}\right) \\
&\leq \frac{p_\phi \phi^2}{222}\frac{221}{\phi^2}  = p_\phi.
\end{split}
\end{equation*}
It follows from Definition \ref{def:regular} that 
\begin{equation}\label{aeqn:using_regular}
\left(1 + \frac{\phi}{3}\right)^{-1}\frac{p(B(x, r_*))}{p(B(x, s_*))} \leq \frac{r_*^k}{s_*^k} \leq \left(1 + \frac{\phi}{3}\right)\frac{p(B(x, r_*))}{p(B(x, s_*)}.\end{equation} 

However, $|S \cap B(x, s_*)| = b$ and $|S \cap B(x, r_*)| = 2b$, which means that we can safely apply our claim to both of these cases. By substituting Equations \ref{aeqn:epsilon_lower_bound} and \ref{aeqn:epsilon_upper_bound} (for both $r_*$, $s_*$) into Equation \ref{aeqn:using_regular}, along with the fact that $\left(1+\frac{\phi}{3}\right)\left(1 + \frac{\phi}{9}\right) \leq \left(1 + \frac{\phi}{2}\right)$, it follows that 
\begin{equation}\label{aeqn:bound_ratio}
\left(1 + \frac{\phi}{2}\right)^{-1} \leq \frac{r_*^k}{2s_*^k} \leq \left(1 + \frac{\phi}{2}\right)
\end{equation}

Finally, by taking logs of Equation \ref{aeqn:bound_ratio} and simplifying, we have that 
\begin{equation*}
\frac{k}{1 + \log_2 \left(1 + \frac{\phi}{2}\right)} \leq \frac{1}{\log_2 \frac{r_*}{s_*}} \leq \frac{k}{1 - \log_2 \left(1 + \frac{\phi}{2} \right)}
\end{equation*}
It consequently suffices to show that $k$ is the unique integer between $\frac{k}{1 + \log_2 \left(1 + 8\epsilon\right)}$ and $\frac{k}{1 - \log_2 \left(1 + 2\epsilon\right)}$. However, this is simply a result of the assumption that $\phi = \frac{1}{2k}$ and standard manipulations, which completes the proof. 
\end{proof}

\section{Proofs}

All proofs to theorems and propositions in the main body are in this section. For each result, we include a restatement for convenience. 

\subsection{Proof of Theorem \ref{thm:KDE}}

We prove a stronger version of Theorem \ref{thm:KDE}.

\begin{theorem}[Theorem \ref{thm:KDE}]
Let $1 < \lambda$ and $\gamma > 0$. Let $\sigma_n$ be a sequence of bandwidths and $K$ be any regular kernel function. For any $n > 0$ there exists a probability distribution $\pi$ with full support over $\R^d$ such for any $S \sim \pi^n$, a KDE trained with bandwidth $\sigma_n$ and kernel function $K$ has data-copy rate $\c_q \geq \frac{1}{2}$.
\end{theorem}

We begin by giving necessary conditions for a kernel $K$ to be regular.

\begin{definition}\label{defn:regular_kernel}
A kernel function, $K: \R^d \to \R_{\geq 0}$ is regular if it satisfies the following conditions.
\begin{enumerate}
	\item $K$ is radially symmetric. That is, there exists $h: \R \to \R$ such that $K(x) = h(||x||)$.
	\item $K$ is regularized. That is, $\int_{\R^d} K(x)dx = 1$.
	\item $K$ decays to $0$. That is, $\lim_{t \to \infty} h(t) = \lim_{t \to -\infty}h(t) = 0$. 
\end{enumerate}
\end{definition}

It is well known that under suitable choices of $\sigma_n$ and several technical assumptions that a regular KDE converges towards the true data distribution in the large sample limit. We now prove Theorem \ref{thm:KDE}.

\begin{proof}
Fix any $n$, and for convenience let denote $\sigma_n$ by $\sigma$. Because $K$ is non-negative, by condition 2. of Definition \ref{defn:regular_kernel}, there exists $R > 0$ such that $\int_{||x|| \leq R} K(x)dx = \frac{1}{2}$. Let $$D = R\sigma \left(\max\left(2n\lambda, \frac{1}{\gamma}\right)\omega_d\right)^{1/d},$$ where $\omega_d$ denotes the volume of the unit ball in $d$ dimensions. We let $\pi$ denote the uniform distribution over $[0, D]^d$, and claim that this suffices. 

Let $S \sim \pi^n$ be a training sample, with $S = \{x_1, x_2, \dots, x_n\}$, and let $q$ be a KDE trained from $S$ with bandwidth $\sigma$ and kernel function $K$. Suppose $x \sim q$ satisfies that $x \in B(x_i, R\sigma)$. We claim that $q$ $(\lambda, \gamma)$-copies $x$.

To see this, it suffices to bound $\pi((B(x_i, R\sigma))$ and $q(B(x_i, R\sigma))$. The former quantity satisfies
\begin{equation*}
\begin{split}
\pi((B(x_i, R\sigma)) &\leq \frac{vol(B(x_i, R\sigma))}{D^d} \\
&= \frac{\omega_d(R\sigma)^d}{D^d} \\
&= \frac{1}{\max\left(2n\lambda, \frac{1}{\gamma}\right)} \\ 
&\leq \min\left(\gamma, \frac{1}{2n\lambda}\right),
\end{split}
\end{equation*}
which implies that the third condition of Definition \ref{defn:data_copy} is met. Meanwhile, $q((B(x_i, R\sigma))$ can be bounded as
\begin{equation*}
\begin{split}
q((B(x_i, R\sigma)) &= \int_{B(x_i, R\sigma)} \frac{1}{n\sigma}\sum_{j = 1}^n K \left(\frac{x - x_j}{\sigma} \right)dx \\
&\geq \int_{B(x_i, R\sigma)} \frac{1}{n\sigma}K \left(\frac{x - x_i}{\sigma} \right)dx \\
&= \int_{||u|| \leq R} \frac{1}{n}K(u)du \\
&\geq \frac{1}{2n},
\end{split}
\end{equation*}
which implies that $q((B(x_i, R\sigma)) \geq \lambda p(B(x_i, R\sigma))$ giving the second condition of Definition \ref{defn:data_copy}. Thus, it follows that $q$ $(\lambda,\gamma)$-copies all $x \in B(x_i, R\sigma)$. It consequently suffices to bound $q\left(\bigcup_{i= 1}^n B(x_i, R\sigma)\right)$. 

To do so, let $\eta$ denote the probability distribution over $\R^d$ with probability density function $\eta(x) = \frac{1}{\sigma}K(\frac{x}{\sigma})$, and let $\hat{q}$ denote the probability density function induced by the following random process:
\begin{enumerate}
	\item Select $1 \leq i \leq n$ at uniform.
	\item Select $x \sim \eta$.
	\item Output $x + x_i$.
\end{enumerate}
The key observation is that $\hat{q}$ has precisely the same density function as $q$ -- $q$s density function is clearly a convolution of selecting $x_i$ and then adding $x \sim \eta$. Applying this, we have
\begin{equation*}
\begin{split}
\Pr_{x \sim q}\left[x \in \bigcup_{i= 1}^n B(x_i, R\sigma)\right] &= \Pr_{x \sim \hat{q}}\left[x \in \bigcup_{j= 1}^n B(x_j, R\sigma)\right] \\
&= \frac{1}{n} \sum_{i=1}^n \Pr_{x \sim \tau}\left[x \in \left(\bigcup_{j= 1}^n B(x_j, R\sigma) - x_i\right)\right]\\
&\geq \frac{1}{n} \sum_{i=1}^n \Pr_{x \sim \tau} \left[x \in \left(B(x_i, R\sigma) - x_i\right)\right] \\
&= \int_{B(0, R\sigma)} \tau(x)dx \\
&= \int_{B(0, R\sigma)} \frac{1}{\sigma} K\left(\frac{x}{\sigma}\right)dx \\
&= \int_{B(0, R)} K(u)du \\
&\geq \frac{1}{2},
\end{split}
\end{equation*}
completing the proof.

\end{proof}

\subsection{Proof of Proposition \ref{prop:manifold_works}}

\begin{proposition}[Proposition \ref{prop:manifold_works}] Let $p$ be a probability distribution with support precisely equal to a smooth, compact, $k$-dimensional sub-manifold of $\R^d$, $M$. Additionally, suppose that $p$ has a continuous density function over $M$. Then it follows that $p$ is $k$-regular.
\end{proposition}

To prove this, we begin with the following lemma.

\begin{lemma}\label{lem:conditions_for_alpha_regularity}
Let $k > 0$ be a constant. Let $p$ be a probability distribution for which the following properties hold:

1. The map $supp(p) \times \R^+ \to R^+$ defined by $(x, r) \mapsto p(B(x, r))$ is continuous.

2. The map $supp(p) \to \R^+$ defined by $x \mapsto \lim_{r \to 0}\frac{p(B(x, r)}{r^k}$ is  well defined, continuous, and strictly positive over its domain.

3. $p$ has compact support. 

Then $p$ is $k$-regular. 
\end{lemma}

\begin{proof}
The map $r \to r^k$ is clearly continuous. It follows by properties (1.) and (2.), the following is a continuous map: $F: supp(p) \times \R^{\geq 0} \to \R^+$ where $$F(x, r) = \begin{cases} \frac{p(B(x, r))}{r^k} & r > 0 \\\lim_{s \to 0} \frac{p(B(x, s))}{s^k} & r = 0,\end{cases}.$$

Next, fix $\alpaca > 0$, as arbitrary. We desire to show that $p_\alpaca$ exists for which the conditions of Definition \ref{def:regular} hold. Without loss of generality, we can assume $\alpaca < 1$, as the case $\alpaca \geq 1$ can easily be handled by just using $p_\alpaca$ for a smaller value of $\alpaca$. 

For any $x > 0$, since $F$ is continuous, there exists $\rho_x > 0$ such that for any $x', \in B(x, \rho_x)$ and $r \leq \rho_x$, $$|F(x', r) - F(x, 0)| < F(x, 0)\frac{\alpaca}{9}.$$ It follows for any such $x'$ that 
\begin{equation}\label{eqn:p_epsilon_works}
\begin{split}
p(B(x', \rho_x)) &= F(x, \rho_x)\rho_x^k \geq (F(x, 0))(1- \frac{\alpaca}{9}),
\end{split}
\end{equation}
and for any $0 < s < r < \rho_x$, we have
\begin{equation*}
\begin{split}
\frac{p(B(x', r))}{r^k} &= F(x', r) \\
&\leq F(x, 0)(1 + \frac{\alpaca}{9}) \\
&\leq F(x', s)\frac{1 + \frac{\alpaca}{9}}{1 - \frac{\alpaca}{9}} \\ 
&\leq F(x', s)\left(1 + \frac{\alpaca}{3}\right),
\end{split}
\end{equation*}
and
\begin{equation*}
\begin{split}
\frac{p(B(x', r))}{r^k} &= F(x', r) \\
&\geq F(x, 0)(1 - \frac{\alpaca}{9}) \\
&\geq F(x', s)\frac{1 - \frac{\alpaca}{9}}{1 + \frac{\alpaca}{9}} \\ 
&\geq F(x', s)\left(1 +\frac{\alpaca}{3}\right)^{-1},
\end{split}
\end{equation*}
which together imply that 
\begin{equation}\label{eqn:it_converges}
\left(1 + \frac{\alpaca}{3}\right)^{-1} \frac{p(B(x, s))}{s^k} \leq \frac{p(B(x, r))}{r^k} \leq \left(1 + \frac{\alpaca}{3}\right)\frac{p(B(x, s))}{s^k}.
\end{equation}
Finally, observe that the balls $B(x, r_x)$ cover the support of $p$. Since $supp(p)$ is compact, it follows that there exists a finite sub-cover of such balls, $C$. We finally let $p_\epsilon = \min_{B(x, r_x) \in C} F(x, 0)(1 - \frac{\alpaca}{9})$. It then follows by Equations \ref{eqn:p_epsilon_works} and \ref{eqn:it_converges}, that $p$ has met the criteria necessary for $p$ to be $k$-regular, as desired. 
\end{proof}

We are now prepared to prove Proposition \ref{prop:manifold_works}. 

\begin{proof}
It suffices to show that the conditions of Lemma \ref{lem:conditions_for_alpha_regularity} hold. Conditions 1. and 3. immediately hold since the probability mass of the surface (i.e. points on the boundary) of a ball $B(x, r)$ will be $0$ as its intersection with $M$ would be a $(k-1)$-dimensional manifold. 

Thus, it remains to verify condition 2. For any $x, y \in M$, let $d_M(x, y)$ denote the geodesic distance between $x$ and $y$ (with $||x - y||$ still denoting their euclidean distance in $\R^d$ as $M$ is embedded in $\R^d$). Since $M$ is a smooth, compact manifold, it follows that for any $x \in M$, $$\lim_{r \to 0} \sup_{||x - y|| = r}\frac{||x- y||}{d_M(x, y)}  = 1.$$ In other words, at a small scale, the geodesic distance and the Euclidean distance converge. It follows that $$\lim_{r \to 0} \frac{p(B(x, r))}{r^k} = \lim_{s \to 0} \frac{p(B_M(x, s))}{s^k},$$ where $B_M(x, s)$ denotes the geodesic ball of radius $s$ centered at $x$ on $M$. However, the latter quantity is precisely equal to the density function over $M$ (up to a constant factor, since $\lim_{s \to 0} \frac{vol_M(B_M(x, s))}{s^k} = \omega_k$, where $\omega_k$ is the volume of the $k$-sphere). Since by assumption our density function is continuous and non-zero everywhere on the manifold, it follows that the map above must be well defined and continuous giving us condition 2. of Lemma \ref{lem:conditions_for_alpha_regularity}, as desired. 
\end{proof}

\subsection{Proof of Proposition \ref{prop:est_works}}\label{sec:est_works}

\begin{proposition}[Proposition \ref{prop:est_works}]
Let $p$ be an $k$-regular distribution, and let $\alpaca >0$ be arbitrary. Then if $n = O\left(\frac{d\ln\frac{1}{\delta \alpaca p_\alpaca}}{\alpaca^2 p_\alpaca}\right)$ with probability at least $1 - \delta$ over $S \sim p^n$, for all $x \in \R^d$ and $r > 0$, \begin{equation}\label{eqn:prop_est_works}\left(1+\frac{\alpaca}{2}\right)^{-1}p(B(x,r))\leq Est(x, r, S) \leq \left(1+\frac{\alpaca}{2}\right)p(B(x,r)).\end{equation}
\end{proposition}

\begin{proof}
We begin by first applying standard uniform convergence over $\ell_2$ balls in $\R^d$ (which have a VC dimension of at most $d+2$). To this end, let $$\beta_n = \sqrt{\frac{4(d+2)\ln \frac{16n}{\delta}}{n}}.$$ Then by the standard result of Vapnik and Chervonenkis, with probability $1-\delta$ over $S \sim p^n$, for all $x \in \R^d$ and all $r > 0$, 
\begin{equation}\label{eqn:VC}
\frac{|S \cap B(x,r)|}{n} - \beta_n\sqrt{\frac{|S \cap B(x,r)|}{n}} \leq p(B(x,r)) \leq \frac{|S \cap B(x,r)|}{n} + \beta_n^2 + \beta_n\sqrt{\frac{|S \cap B(x,r)|}{n}}.
\end{equation}

Next, assume that 
\begin{equation}\label{eqn:bound_n}
n \geq \frac{888(d+2)\ln \left(\frac{14208(d+2)}{\delta \min(\alpaca,1)^2 p_\alpaca} \right)}{\min(\alpaca,1)^2 p_\alpaca}.
\end{equation} 
It is clear that for an appropriate constant, we have $n = O \left(\frac{d\ln\frac{d}{\delta\alpaca p_\alpaca}}{\alpaca^2 p_\alpaca} \right)$. Thus, it suffices to show that if Equation \ref{eqn:VC} holds for all $x, r$, then the desired bound,  Equation \ref{eqn:prop_est_works}, does as well. 

To this end, let $x, r$ be arbitrary, and let $b$ be as defined in Algorithm \ref{alg:estimate}. Let $b' = |S \cap B(x, r)|$ be the number of elements from $S$ in $B(x, r)$. Then we have two cases.

\textbf{Case 1: $b' \geq b$}

It follows from Algorithm \ref{alg:estimate} that $Est(x, r, S) = \frac{b'}{n}$. We now set $b$ as  \begin{equation}\label{eqn:defining_b}\frac{b}{n} = \frac{400(d+2)\ln\frac{16n}{\delta}}{n\min(\alpaca, 1)^2} = \frac{100\beta_n^2}{\epsilon^2},\end{equation} which clearly obeys the desired asymptotic bound given in Algorithm \ref{alg:estimate}. Let $c = \sqrt{\frac{b'}{n\beta_n^2}}$. Then $b' \geq b$ implies that $c \geq \frac{10}{\min(\alpaca, 1)}$. It follows that \begin{equation}\label{eqn:c_stuff}\frac{c+1}{c^2} \leq \frac{1}{c-1} \leq \frac{\min(\alpaca, 1)}{9}.\end{equation} Substituting Equations \ref{eqn:defining_b} and \ref{eqn:c_stuff} into  Equation \ref{eqn:VC}, we have 
\begin{equation}\label{eqn:epsilon_lower_bound}
\begin{split}
\frac{Est(x, r, S)}{p(B(x, r))} &\geq \frac{\frac{b'}{n}}{\frac{b'}{n} + \beta_n^2 + \beta_n \sqrt{\frac{b'}{n}}} \\
&= \frac{c^2}{c^2 + 1 + c} \\
&= \left(1 + \frac{c+1}{c^2} \right)^{-1} \\
&\geq \left(1 + \frac{\min(\alpaca, 1)}{9}\right)^{-1}
\end{split}
\end{equation}
and 
\begin{equation}\label{eqn:epsilon_upper_bound}
\begin{split}
\begin{split}
\frac{Est(x, r, S)}{p(B(x, r))} &\leq \frac{\frac{b'}{n}}{\frac{b'}{n} - \beta_n \sqrt{\frac{b'}{n}}} \\
&= \frac{c^2}{c^2 - c} \\
&= 1 + \frac{1}{c-1} \\
&\leq 1 + \frac{\min(\alpaca, 1)}{9}.
\end{split}
\end{split}
\end{equation}
Together, Equations \ref{eqn:epsilon_lower_bound} and \ref{eqn:epsilon_upper_bound} imply that $Est(x, r, S)$ is sufficiently accurate.

\textbf{Case 2: $b' < b$ }

We begin by bounding $\beta_n^2$ in terms of $p_\epsilon$. We have, 
\begin{equation}\label{eqn:bound_alpha_n}
\begin{split}
\beta_n^2 &= \frac{4(d+2)\ln(16n/\delta)}{n} \\
&= 4(d+2) \ln \left(\frac{14208(d+2)}{\delta \min(\alpaca,1)^2 p_\alpaca}\ln\left(\frac{14208(d+2)}{\delta \min(\alpaca,1)^2 p_\alpaca} \right)\right) \frac{\min(\alpaca,1)^2 p_\alpaca}{888(d+2)\ln \left(\frac{14208(d+2)}{\delta \min(\alpaca,1)^2 p_\alpaca} \right)} \\
&\leq 8(d+2)\ln \left(\frac{14208(d+2)}{\delta \min(\alpaca,1)^2 p_\alpaca}\right) \frac{\min(\alpaca, 1)^2 p_\alpaca}{888(d+2)\ln \left(\frac{14208(d+2)}{\delta \min(\alpaca,1)^2 p_\epsilon} \right)} \\
&= \frac{p_\alpaca \min(\alpaca, 1)^2}{111}.
\end{split}
\end{equation}
Now, let $r_*$ be as defined in Algorithm \ref{alg:estimate}. Then $|S \cap B(x, r_*)| = b$. Our main idea will be to show that $p(B(x, r_*) \leq p_\alpaca$, and then use Equations \ref{eqn:epsilon_lower_bound} and \ref{eqn:epsilon_upper_bound} for $r_*$ (which is possible since $|S \cap B(x, r_*)| = b$) along with the definition of $p_\alpaca$ (Definition \ref{def:regular}) to bound $Est(x, r, S)$ in terms of $p(B(x, r))$. To this end, we have by Equations \ref{eqn:VC} and \ref{eqn:bound_alpha_n} along with the fact that $b = \frac{100\beta_n^2}{\min(\alpaca, 1)^2}$ (Equation \ref{eqn:defining_b}) that 
\begin{equation*}
\begin{split}
p(B(x, r_*)) &\leq \frac{|S \cap B(x, r_*)|}{n} + \beta_n^2 + \beta_n \sqrt{\frac{|S \cap B(x, r_*)|}{n}} \\
&= \frac{b}{n} + \beta_n^2 + \beta_n\sqrt{\frac{b}{n}} \\
&= \beta_n^2\left(\frac{100}{\min(\alpaca,1)^2} +1 + \frac{10}{\min(\alpaca,1)}\right) \\
&\leq \frac{p_\alpaca^2 \min(\alpaca, 1)^2}{111}\frac{111}{\min(\alpaca, 1)^2}  = p_\alpaca.
\end{split}
\end{equation*}
It follows from Definition \ref{def:regular} that 
\begin{equation}\label{eqn:using_regular}
\left(1 + \frac{\alpaca}{3}\right)^{-1}\frac{p(B(x, r_*))r^k}{r_*^k} \leq p(B(x, r)) \leq \left(1 + \frac{\alpaca}{3}\right)\frac{p(B(x, r_*))r^k}{r_*^k}.\end{equation} 
Finally, by the definition of $Est(x, r, S)$) (Algorithm \ref{alg:estimate}), we have that $Est(x, r, S) = \frac{Est(x, r_*, S)r^k}{r_*^k}$. Combining this with Equation \ref{eqn:using_regular} the definition of $Est(x,r, S)$ (Algorithm \ref{alg:estimate}) along with Equations \ref{eqn:epsilon_lower_bound} and \ref{eqn:epsilon_upper_bound} (which can be safely applied to $r_*$ by reverting to Case 1), we have
\begin{equation*}
\begin{split}
\frac{Est(x, r, S)}{p(B(x, r))} &= \frac{\frac{Est(x, r_*, S)r^k}{r_*^k}}{p(B(x, r))} \leq \frac{\frac{Est(x, r_*, S)r^k}{r_*^k}\left(1 + \frac{\alpaca}{3}\right)}{\frac{p(B(x, r_*))r^k}{r_*^k}} \\
&= \frac{Est(x, r_*, S)\left(1 + \frac{\alpaca}{3}\right)}{p(B(x, r_*))} \leq \left(1 + \frac{\alpaca}{3}\right)\left(1 + \frac{\min(\alpaca, 1)}{9}\right) \\
&\leq 1 + \frac{\alpaca}{2},
\end{split}
\end{equation*}
and
\begin{equation*}
\begin{split}
\frac{Est(x, r, S)}{p(B(x, r))} &= \frac{\frac{Est(x, r_*, S)r^k}{r_*^k}}{p(B(x, r))} \geq \frac{\frac{Est(x, r_*, S)r^k}{r_*^k}}{\frac{p(B(x, r_*))r^k}{r_*^k}\left(1 + \frac{\alpaca}{3}\right)} \\
&= \frac{Est(x, r_*, S)}{p(B(x, r_*))\left(1 + \frac{\alpaca}{3}\right)} \geq \left(1 + \frac{\alpaca}{3}\right)^{-1}\left(1 + \frac{\min(\alpaca, 1)}{9}\right)^{-1} \\
&\geq \left(1 + \frac{\alpaca}{2}\right)^{-1},
\end{split}
\end{equation*}
which concludes the proof.

\end{proof}

\subsection{Proof of Theorem \ref{thm:upper_bound}}\label{sec:upper_bound_proof}

\begin{theorem}[Theorem \ref{thm:upper_bound}]
$\dc{}$ is a data-copying detector (Definition \ref{def:data_copy_detector}) with sample complexity at most $$m_p(\epsilon, \delta) = O\left(\frac{d\ln\frac{d}{\delta\alpaca p_\alpaca}}{\alpaca^2 p_\alpaca}\right),$$ for all regular distributions, $p$. 
\end{theorem}

\begin{proof}

Let $C$ be the constant defined in Proposition \ref{prop:est_works}, and let $n \geq C \frac{d \ln \frac{d}{\delta \epsilon p_\epsilon}}{\epsilon^2p_\epsilon}.$ Let $S \sim p^n$ be a set of $n$ i.i.d training points, $\{x_1, x_2, \dots, x_n\}$, and let $q \sim A(S)$ be an arbitrary generated distribution. 

By Proposition \ref{prop:est_works}, the subroutine $Est(x, r, S)$ is accurate over any $x$ and $r$ up to a factor of $(1 + \alpaca)$ with probability at least $1-\frac{\delta}{3}$ (we can achieve this by simply making $n$ a bit larger and substituting $\frac{\delta}{3}$ into Proposition \ref{prop:est_works}). Suppose this holds, meaning that that for all $x \in \R^d$ and all $r > 0$, the condition of Proposition \ref{prop:est_works} holds, and 
\begin{equation}\label{eqn:est_is_accurate}
(1+\alpaca)^{-1}p(B(x,r)) \leq Est(x,r, S) \leq (1+\epsilon)p(B(x, r)).
\end{equation} 
We desire to show that $$\c_q^{-\alpaca}  - \epsilon \leq DataCopyDetect(q, S) \leq \c_q^{\alpaca} + \epsilon.$$ 

To do so, we begin applying uniform convergence over $T \sim q^m$. To this end, let $$\beta_m = \sqrt{\frac{4(d+2)\ln \frac{48m}{\delta}}{m}}.$$ Then by the standard result of Vapnik and Chervonenkis, with probability $1-\frac{\delta}{3}$ over $T \sim q^m$, for all $x \in \R^d$ and all $r > 0$, 
\begin{equation}\label{qeqn:VC}
\frac{|T \cap B(x,r)|}{m} - \beta_m\sqrt{\frac{|T \cap B(x,r)|}{m}} \leq q(B(x,r)) \leq \frac{|T \cap B(x,r)|}{m} + \beta_m^2 + \beta_n\sqrt{\frac{|T \cap B(x,r)|}{m}}.
\end{equation}

Observe that by the definition of $m$, we have 
\begin{equation}\label{qeqn:bound_beta_n}
\begin{split}
\beta_m^2 &= \frac{4(d+2)\ln(48m/\delta)}{m} \\
&= 4(d+2) \ln \left(\frac{98304n^2(d+2)}{\delta\epsilon^2\min(\alpaca, 1)^2}\ln \left(\frac{98304n^2(d+2)}{\delta \epsilon^2\min(\alpaca, 1)^2}\right)\right) \frac{\epsilon^2\min(\alpaca, 1)^2}{2048n^2(d+2)\ln \left(\frac{98304n^2(d+2)}{\delta \epsilon^2\min(\alpaca, 1)^2} \right)}  \\
&\leq 8(d+2)\ln \left(\frac{98304n^2(d+2)}{\delta \epsilon^2\min(\alpaca, 1)^2}\right) \frac{\epsilon^2\min(\alpaca, 1)^2}{2048n^2(d+2)\ln \left(\frac{98304n^2(d+2)}{\delta \epsilon^2\min(\alpaca, 1)^2} \right)} \\
&= \frac{\epsilon^2\min(\alpaca, 1)^2}{256n^2}.
\end{split}
\end{equation}

Next, suppose $x, r$ satisfy that $q(B(x, r)) \geq \frac{\epsilon}{2n}$. For convenience, let $\widehat{q(B(x, r))}$ denote $\frac{|T \cap B(x, r)|}{m}$. By applying Equations \ref{qeqn:VC} and \ref{qeqn:bound_beta_n}, it follows that 
\begin{equation*}
\begin{split}
\frac{\hq}{q(B(x, r))} &\leq  \frac{q(B(x, r)) + \beta_m}{q(B(x, r)}\\
&\leq 1 + \frac{\beta_m}{q(B(x, r)} \\
&\leq 1 + \frac{\min(\alpaca, 1)}{8},
\end{split}
\end{equation*}
and
\begin{equation*}
\begin{split}
\frac{q(B(x, r))}{\hq} &\leq  \frac{q(B(x, r))}{q(B(x, r)) - \beta_m^2 - \beta_m\sqrt{\hq}}\\
&\leq \frac{q(B(x, r))}{ q(B(x, r) - 2\beta_m} \\
&= \frac{1}{1 - \frac{2\beta_m}{q(B(x, r))}} \\
&\leq \frac{1}{1 - \frac{\min(\alpaca, 1)}{4}} \\
&\leq 1 + \frac{\min(\alpaca, 1)}{3}.
\end{split}
\end{equation*}

Combining these, we have 
\begin{equation}\label{eqn:bound:q:hat}
\left(1 + \frac{\min(\alpaca, 1)}{3}\right)^{-1} \leq \frac{q(B(x, r)}{\hq} \leq \left(1 + \frac{\min(\alpaca, 1)}{3}\right)
\end{equation}

Next, for $1 \leq i \leq n$, let $r_i^*$ be the radii defined in Algorithm \ref{alg:main}. Define $r_i^{-\alpaca}$ and $r_i^{\alpaca}$ to be the maximal radii $r$ for which $q$ respectively $(\lambda(1+\alpaca), \gamma(1+\alpaca)^{-1})$-copies, and $(\lambda(1+\alpaca)^{-1}, \gamma(1+\alpaca))$-copies $p$ about $x_i$. Then we have the following claims.

\textbf{Claim 1:} For $1 \leq i \leq n$, if $q(B(x, r_i^*)) \geq \frac{\epsilon}{2n}$, $r_i^* \leq r_i^{\alpaca}$. 

\begin{proof}
Because $Est(x_i, r_i^*, S) \leq \gamma$, it follows by Equation \ref{eqn:est_is_accurate} that $p(B(x_i, r_i^*)) \leq \left(1+\frac{\alpaca}{2}\right)\gamma$. Furthermore, by also applying Equation \ref{eqn:bound:q:hat}  we have that 
\begin{equation*}
\frac{q(x_i, r_i^*)}{p(x_i, r_i^*)} \geq \frac{\frac{|B(x_i, r_i^*) \cap T|}{m}}{Est(x_i, r_i^*, S)\left(1 + \frac{\min(\alpaca, 1)}{3}\right)\left(1 + \frac{\alpaca}{2}\right)} \geq \lambda (1+\alpaca)^{-1}.
\end{equation*}
Thus $q$ $(\lambda(1+\alpaca)^{-1}, \gamma(1+\alpaca))$-copies all points in $B(x_i, r_i^*)$ implying $r_i^* \leq r_i^\alpaca$. 
\end{proof}

\textbf{Claim 2:} For $1 \leq i \leq n$, if $q(B(x, r_i^{-\alpaca})) \geq \frac{\epsilon}{2n}$, then $r_i^{-\alpaca} \leq r_i^*$. 

\begin{proof}
For the left hand side, we use a similar argument. By Equation \ref{eqn:est_is_accurate} along with the definition of $r_i^\alpaca$, we have $Est(x_i, r_i^{-\alpaca}, S) \leq \gamma(1+\alpaca)^{-1} \left(1 + \frac{\alpaca}{2}\right) \leq \gamma$. By Equations \ref{eqn:est_is_accurate} and \ref{eqn:bound:q:hat}, we have
\begin{equation*}
\frac{\frac{|B(x_i, r_i^{-\alpaca}) \cap T|}{m}}{Est(x_i, r_i^*, S)} \geq  \frac{q(B(x_i, r_i^{-\alpaca}))}{p(B(x_i, r_i^{-\alpaca}))\left(1 + \frac{\min(\alpaca, 1)}{3}\right)\left(1 + \frac{\alpaca}{2}\right)} \geq \lambda,
\end{equation*}
with the last inequality coming again from the definition of $r_i^{-\alpaca}$. Thus, $r_i^{-\alpaca}$ meets the criteria from Algorithm \ref{alg:main} required to be selected as $r_i^*$. As a technical note, because Algorithm \ref{alg:main} only considers finitely many radii, it may not consider precisely $r_i^{-\alpaca}$. However, this is not a problem, as the nearest considered radii to this point have nearly unchanged values of $Est(x, r, S)$ and $\frac{|B(x, r) \cap T|}{m}$, meaning that some similar radius will be considered. 
\end{proof}

Finally, armed with our claims, we now consider the total region of points in which Algorithm \ref{alg:main} claimed data-copying occurs. Let $S^1$ and $S^2$ be the sets of indices for which the conditions are violated for claims $1$ and $2$ respectively. Then it follows from Claim 1 that 
\begin{equation*}
\begin{split}
\c_q^\alpaca - q \left(\cup_{i= 1}^n B(x_i, r_i^*) \right) &= q \left(\cup_{i= 1}^n B(x_i, r_i^\alpaca)\right) - q \left(\cup_{i= 1}^n B(x_i, r_i^*)\right) \\
&\geq q \left(\cup_{i= 1}^n B(x_i, r_i^\alpaca)\right) - q \left(\cup_{i \notin S^1} B(x_i, r_i^*)\right) - q \left(\cup_{i \in S^1} B(x_i, r_i^*)\right) \\
&\geq -\frac{\epsilon}{2}.
\end{split}
\end{equation*}
Here we are using Claim 1 to hand all terms that are not in $S^1$, and then crudely bounding the remaining terms with $\frac{\epsilon}{2n}$. Similarly, by Claim 2, we have 
\begin{equation*}
\begin{split}
q \left(\cup_{i= 1}^n B(x_i, r_i^*) \right) - \c_q^{-\alpaca} &= q \left(\cup_{i= 1}^n B(x_i, r_i^*)\right) - q \left(\cup_{i= 1}^n B(x_i, r_i^{-\alpaca})\right) \\
&\geq q \left(\cup_{i= 1}^n B(x_i, r_i^*)\right) - q \left(\cup_{i \notin S^2} B(x_i, r_i^{-\alpaca})\right) - q \left(\cup_{i \in S^2} B(x_i, r_i^{-\alpaca})\right) \\
&\geq -\frac{\epsilon}{2}.
\end{split}
\end{equation*}
Combining these, we see that $$\c_q^{-\alpaca} - \frac{\epsilon}{2} \leq q\left(\cup_{i=1}^n B(x_i, r_i^*) \right) \leq \c_q^{\alpaca} + \frac{\epsilon}{2}.$$ All the remains is to show that our last step of Algorithm \ref{alg:main}, in which we estimate this mass, is accurate up to a factor of $\frac{\epsilon}{2}$. However, this immediately follows from the fact that we use $\frac{20\log \frac{1}{\delta}}{\epsilon^2}$ samples (last line of Algorithm \ref{alg:main}). In particular, because this holds with probability $1- \frac{\delta}{3}$, we can apply a union bound with our other two probabilistic events ($Est$ being sufficiently close, and $T$ yielding uniform convergence)  to get a total failure probability of $\delta$, as desired. 
\end{proof}

\subsection{Proof of Theorem \ref{thm:lower_bound}}

\begin{theorem}[Theorem \ref{thm:lower_bound}] Let $B$ be a data-copying detector. Let $\epsilon = \delta = \frac{1}{3}$. Then there exist $1$-regular distributions $p$ for which $p_\epsilon$ is arbitrarily small and $B$ has sample complexity $$m_p(\epsilon, \delta) \geq \Omega(\frac{1}{p_\epsilon}).$$ More precisely, for all integers $\a > 0$, there exists a probability distribution $p$ such that $\frac{1}{9\a} \leq p_\epsilon \leq \frac{1}{\a}$, and $m_p(\epsilon, \delta) > \Omega(\a).$ 
\end{theorem}

\textbf{Proof Outline:} Let $\a$ be a sufficiently large integer. Then we take the following steps.
\begin{enumerate}
	\item We define the probability distribution $p_T$, where $T \subset [2\a] = \{1, 2, \dots, 2\a\}$ is a subset with $|T|=\a$ that parametrizes our distribution. We then show that for all $T$, $p_T$ is a $1$-regular distribution satisfying $\frac{1}{9\a} \leq (p_T)_\epsilon \leq \frac{1}{\a}$. 
	\item We define a generative algorithms $A_T$ and $A_T'$, where as before $T \subset [2\a]$ with $|T| = \a$. We then show that if $S \sim p_T^{O(\a)}$, $A_T(S)$ is likely to have a high data-copy rate with respect to $p_T$, whereas $A_T'(S)$ has a data-copy rate of $0$.
	\item We construct families $$\mathcal{F} =\{(p_T, A_T): T \subset [2\a], |T| = \a\}\text{ and }\mathcal{F}' = \{(p_T, A_T'): T \subset [2\a], |T| = \a\},$$ and show that $(S, A(S))$ follows very similar distributions when $S$ is drawn from $p^{O(\a)}$ and $(p, A)$ is drawn from $\mathcal{F}$ and $\mathcal{F}'$ respectively, meaning that it is difficult to tell which family the pair $(p, A)$ is drawn from. 
	\item We show that if $B$ has sample complexity at most $O(\a)$, then by (2.) it would be able to distinguish $(S, A_T(S))$ from $(S, A_T'(S))$ thus contradicting (3.) We thus conclude $B$ has sample complexity $\Omega(\a)$, as desired. 
\end{enumerate}

\begin{proof}
We follow the outline above proceeding step by step. 

\textbf{Step 1: constructing $p_T$}

First, set $\gamma < 1$ arbitrarily, and let $\lambda = 13$. Note that these constants are chosen out of convenience, and for different values of $\epsilon, \delta$, different ones can be chosen. 

Let $\a > 0$ be any integer, and let $[2\a] = \{1, 2, 3, \dots, 2\a\}$. Let $C_1, C_2, \dots, C_{2\a}$ be $2\a$ disjoint unit circles in $\R^d$ with distance at least $3$ between any two circles. All data distributions, $p_T$, that we construct will have support over $\cup_{i = 1}^{2\a} C_i$, and will further obey the constraint that their marginal distribution over any $C_i$ is precisely the uniform distribution. Thus, a distribution $p_T$ is uniquely specified by the probability mass it assigns to each circle. To this end, we define $p_T$ as follows. 

\begin{definition}
Let $T \subset [2\a]$ be a subset of indices with $|T| = \a$. Then $p_T$ is the unique probability distribution satisfying the criteria above such that $$p_T(C_i) = \begin{cases} \frac{1}{3\a} & i \in T \\ \frac{2}{3\a} & i \notin T \end{cases}$$
\end{definition}

\begin{lemma}\label{lem:p_T_is_regular}
$p_T$ is $1$-regular, and satisfies $\frac{1}{9\a} \leq (p_T)_\epsilon \leq \frac{2}{3\a}$ when $\epsilon = \frac{1}{3}$. 
\end{lemma}

\begin{proof}
First, we observe that by Proposition \ref{prop:manifold_works}, we immediately have that $p_T$ is $1$-regular as a union of disjoint circles is a $1$ dimensional closed manifold, and the density function of $p_T$ with respect to each circle is uniform and therefore continuous. For convenience, we let $p$ denote $p_T$, as by symmetry, $(p_T)_\epsilon$ is equal for all values of $T$. 

Next, for $r \leq 2$ and $x \sim p$, we compute $\frac{p(B(x, r)}{r}$. Suppose $x \in C_i$. The key observation is that the density of $p$ over $C_i$ is uniform, and thus since $r \leq 2$, the mass of $B(x, r)$ can be found by simply computing the arc length. It follows that \begin{equation}\label{eqn:ball_mass}\frac{p(B(x, r))}{r} = p(C_i))\frac{4\arcsin(\frac{r}{2})}{2\pi r}.\end{equation} 

By some basic properties about $\arcsin$, it follows that $\frac{p(B(x, r)}{r}$ is monotonically increasing with $0 < r \leq 2$ and satisfies $\lim_{r \to 0^+} \frac{p(B(x, r)}{r} = \frac{p(C_i)}{\pi}$ and $\frac{p(B(x, 2))}{2} = \frac{p(C_i)}{2}$. Using this, we now prove the upper and lower bounds for $p_\epsilon$ beginning with the upper bound.

Assume towards a contradiction that $p_\epsilon > \frac{2}{3\a}$. By Definition \ref{def:regular}, this implies that for any sufficiently small $r > 0$, we have $$\left(1+ \frac{\epsilon}{3} \right)^{-1}\frac{p(B(x, r))}{r} \leq \frac{p(B(x, 2))}{2} \leq \left(1+ \frac{\epsilon}{3}\right)\frac{p(B(x, r))}{r},$$ as for any $x \sim p$, $p(B(x, 2)$ is at most $\frac{2}{3\a}$. Substituting equation \ref{eqn:ball_mass} and taking the limit as $r \to 0^+$, it follows that $\frac{p(C_i)}{2} \leq \frac{7}{6}\frac{p(C_i)}{\pi},$ which is a contradiction giving us that $p_\epsilon \leq \frac{2}{3\a}$.

Next, for the lower bound, it suffices to show that for any $x$ and any $0 < s \leq r$ with $p(B(x, r)) \leq \frac{1}{9\a}$ that 

\begin{equation}\label{eqn:desired}
\left(1+ \frac{\epsilon}{3} \right)^{-1}\frac{p(B(x, s))}{s} \leq \frac{p(B(x, r))}{r} \leq \left(1+ \frac{\epsilon}{3} \right)\frac{p(B(x, s))}{s}.
\end{equation}

Applying Equation \ref{eqn:ball_mass} with $r = 1$, we have for any $x \sim p$,
\begin{equation*}
\frac{p(B(x, 1))}{1} = p(C_i)\frac{4 \arcsin(\frac{1}{2})}{2\pi} = p(C_i)\frac{1}{3} \geq \frac{1}{3\a}\frac{1}{3} = \frac{1}{9\a}.
\end{equation*}
Since $\frac{p(B(x, r)}{r}$ is monotonic in $r$, it follows that $p(B(x, r)) \leq \frac{1}{9\a}$ only if $r \leq 1$. We are now prepared to prove Equation \ref{eqn:desired}.

The left inequality immediately holds since $\frac{p(B(x, r)}{r}$ is monotonic in $r$. For the right inequality, we have that if $r$ satisfies $p(B(x, r) \leq \frac{1}{9\a}$, then $r \leq 1$ implying for $x \in C_i$,
\begin{equation*}
\begin{split}
\frac{p(B(x, r))}{r} &\leq \frac{p(B(x, 1)}{1} \\
&= p(C_i)\frac{1}{3} \\
&\leq (1 + \frac{1}{9})\frac{p(C_i)}{\pi} \\
&= \left(1 + \frac{\epsilon}{3} \right)\lim_{t \to 0} \frac{p(B(x, t))}{t} \\
&\leq \left(1 + \frac{\epsilon}{3} \right)\frac{p(B(x, s))}{s},
\end{split}
\end{equation*}
as desired. 
\end{proof}

\textbf{Step 2: defining $A_T$ and $A_T'$}

Having defined our probability distributions, $p_T$, we now define our generative algorithms $A_T$ and $A_T'$. Recall that a generative algorithm, $A$, is any process that takes as input a set of points $S \in \R^d$ and then returns a probability distribution, $A(S)$ over $\R^d$. The algorithm is allowed to have randomization. 

$A_T$ and $A_T'$ will always be constrained to output distributions that are similar to $p_T$ in the sense that they have support over a disjoint union of circles, and their marginal distribution over any circle (within the support) is the uniform distribution. The only change is that we add one extra circle, $C_0$, that satisfies $$||C_0 - C_i|| \geq 2 + \max_{i, j} ||C_i - C_j||,$$ meaning that it is very far from all $C_i$. Thus, any outputted distribution by $A_T$ or $A_T'$ can be specified by specifying the probability mass it assigns to each circle in $\{C_0, C_1, \dots, C_{2\a}\}$.

Both $A_T$ and $A_T'$ will operate under the assumption that the training sample of points $S$ is relatively well behaved. In the event that this does not hold, $A_T$ and $A_T'$ will output the uniform distribution over $C_0$ as a default. We now formally define this criteria upon $S$.

\begin{definition}\label{defn:covers}
Let $S$ be a finite set of points and $T \subset [2\a]$ be a set of indices with $|T| = \a$. We say that $S$ covers $T$ the sets $L = \{i: i \in T, |C_i \cap S| = 1\}$ and $L' = \{i: i \notin T, |C_i \cap S| = 1\}$ both satisfy $|L|, |L'| \geq \frac{\a}{8}$.
\end{definition}

Observe that this definition if symmetric with respect to complements meaning that $S$ covers $T$ if and only if $S$ covers $[2\a] \setminus T$. We now use this to define $A_T$ and $A_T'$ beginning with $A_T$.

\begin{definition}\label{defn:a_t}
Let $T \subset[2\a]$ be a subset of indices with $|T| = \a$, and let $S$ be any set of points in $\R^d$. Then $A_T$ consists of the following steps. We let $q$ denote its output, and $A_T(S)$ denote the full distribution of potential generated distributions $q$. 
	\begin{enumerate}
		\item If $S$ does \textit{not} cover $T$, then output the uniform distribution over $C_0$ as $q$.
		\item Otherwise, let $L = \{i: i \in T, |C_i \cap S| = 1\}$ be as defined in Definition \ref{defn:covers}.
		\item Randomly select $L_* \subset L$ with $|L_*| = \frac{\a}{8}$ at uniform.
		\item We then let $q$ be the unique probability distribution satisfying the criteria above with $$q(C_i) = \begin{cases} \frac{\lambda(1+\epsilon)}{3\a} & i \in L_* \\ 0 & i \in [2\a] \setminus L_* \\ 1 - \frac{\lambda(1 + \epsilon)}{24}  & i = 0 \end{cases}$$
	\end{enumerate}
\end{definition}

Having defined $A_T$, we define $A_T'$ by having $A_T' = A_{[2\a] \setminus T}$. That is, 

\begin{definition}\label{defn:a_t_prime}
Let $T \subset[2\a]$ be a subset of indices with $|T| = \a$, and let $S$ be any set of points in $\R^d$. Then $A_T'(S)$ is precisely $A_{[2\a] \setminus T}(S)$ where $[2\a] \setminus T$ is the complement of $T$. 
\end{definition}

Observe that if $S$ covers $T$, then by Definitions \ref{defn:a_t} and \ref{defn:a_t_prime}, $A_T(S)$ and $A_T'(S)$ will both have supports non-trivially intersecting the set of circles over which $p_T$ is based, $\cup_{i=1}^{2\a} C_i$. We now show that this condition is sufficient for our desired behavior with respect to data-copying. 

\begin{lemma}\label{lem:data_copy_bound}
Let $\a$ satisfy $\frac{1}{3\a} \leq \gamma$. For any $T \subset [2\a]$, let $S$ be any set of points in the support of $p_T$ that covers $T$. Then with probability $1$ over the randomness of $A_T$ and $A_T'$, $q_T \sim A_T(S)$ and $q_T' \sim A_T'(S)$ have respective data-copy rates $\c_{q_T}^{-\epsilon}$ and $\c_{q_T'}^{\epsilon}$ satisfying $$\c_{q_T}^{-\epsilon} \geq \frac{\lambda(1 + \epsilon) }{24},$$ $$\c_{q_T'}^{\epsilon} = 0.$$
\end{lemma}

\begin{proof}
Let $L$ and $L'$ be as in Definition \ref{defn:covers}. We begin with $\c_{q_T}^{-\epsilon}$, which was the data-copy rate of $q_T$ with parameters $(\lambda(1+\epsilon), \gamma(1-\epsilon))$ (Definition \ref{defn:approx_data_copy_rate}).

Since $|L| \geq \frac{\a}{8}$, there exists $L_* \subset L$ with $|L_*| = \frac{\a}{8}$ such that $q_T$ has support over $C_0 \cup \{C_i\}_{i \in L_*}$. For any $i \in L_*$, let $x_i$ denote the unique point in the intersection of $C_i$ and $S$. Observe that by the definition of $L$, $p_T(B(x_i, 2)) = \frac{1}{3\a}$. On the other hand, we have $q_T(B(x_i, 2)) = q_T(C_i) = \frac{\lambda(1+\epsilon)}{3\a}$, with the first equality holding since $C_i$ is the only circle that intersects $B(x_i, 2)$. It follows by Definition \ref{defn:data_copy} that $q_T$ $(\lambda(1+\epsilon), \gamma(1+\epsilon)^{-1})$-copies all $x \in C_i$. Taking the total measure (under $q_T$), we have $$\c_{q_T}^{-\epsilon} \geq q_T (\cup_{i \in L_*} C_i) = \frac{\a}{8}\frac{\lambda(1+\epsilon)}{3\a} = \frac{\lambda(1+\epsilon)}{24\a}$$ as desired. 

Next, we show $\c_{q_T'}^{\epsilon} = 0$. To do so, it suffices to show that for all $x \in S$ and $r > 0$, $$q_T'(B(x, r)) < \lambda(1+\epsilon)^{-1} p_T(B(x, r)),$$  as this would imply that no points are $(\lambda(1+\epsilon)^{-1}, \gamma(1+\epsilon))$-copied. 

Observe that  $M = \cup_{1 \leq i \leq 2\a} C_i$ is a $1$-dimensional manifold containing the entire support of $p_T$, and that furthermore the marginal distribution of $q_T'(S)$ over $M$ has a well defined probability density with respect to $M$. Since $x \in S$ and $S \subset M$ (as $S \subset supp(p_T)$), we can consider two cases: if $B(x, r)$ intersects $C_0$ (the only region in the support of $A_T'(S)$ outside $M$), and if $B(x,r)$ does not intersect $C_0$.

\paragraph{Case 1: $B(x, r)$ intersects $C_0$} Observe that by the definition of $C_0$, $C_i \subset B(x, r)$ for all $1 \leq i \leq 2\a$. This is because $C_0$ is very far from all the other circles. However, this implies $M \subset B(x, r)$ meaning that $p_T(B(x, r)) \geq p_T(M) = 1$. However, $q_T'(B(x, r))$ is clearly at most $1$, making the desired inequality trivially hold as $\lambda(1+\epsilon)^{-1} > 1$.

\paragraph{Case 2: $B(x, r)$ does not intersect $C_0$} Observe that this implies $supp(p_T) \cap B(x, r) = supp(q_T' \cap B(x, r) \subseteq M$, as both of these distributions only have support on $M$ when outside of $C_0$. Since $p_T$ and $q_T'$ both have well defined probability densities over $M$, their masses over $B(x, r)$ can be found by integrating their densities over this region. 

However, by the definition of $A_T'$, for any $y \in supp(q_T')$, we have that $y \in C_i$ where $i \in [2\a] \setminus T$. By letting $p_T$ and $q_T'$ denote their respective density functions, it follows that $$p_T(y) = \frac{2}{3\a(2\pi)},\text{ and }q_T'(y) = \frac{\lambda(1+\epsilon)}{3\a(2\pi)}.$$ It follows that $\frac{q_T'(y)}{p_T(x)} = \frac{\lambda(1 + \epsilon)}{2} < \lambda(1+\epsilon)^{-1}.$ Thus, it follows from integrating as $y$ goes over $B(x, r)$ that $q_T'(B(x, r)) < \lambda(1+\epsilon)^{-1}p_T(B(x, r))$ as desired. 

As a slight technical detail, while this inequality will no longer be strict if $p_T(B(x, r)) = 0$, we know that this is never the case since $p_T(B(x, r))$ is strictly positive for all $x \in M$. 

\end{proof}

Next, we bound the probability that set of $\a$ points drawn i.i.d. from $p_T$, $S \sim p_T^\a$, will cover $T$. To do so, we begin with a combinatorial lemma. 

\begin{lemma}\label{lem:combinatorics}
Let $m, n$ be an integers with $\frac{n}{4} \leq m \leq \frac{3n}{4}$. Suppose $m$ numbers are chosen uniformly at random from $\{1, 2, \dots, n\}$. Then with probability at least $1 - 2\exp\left(\frac{-n}{2048}\right)$, at least $\frac{n}{8}$ numbers in $\{1, 2, \dots, n\}$ are selected exactly once.
\end{lemma}

\begin{proof}
Let $b_1, b_2, \dots, b_m$ denote our $m$ numbers chosen from $\{1, 2, \dots, n\}$. For $1 \leq i \leq m$, let $X_i$ be an indicator variable for $b_i$ being distinct from $x_j$ for all $1 \leq j < i$, and let $Y_i = 1 - X_i$ be an indicator variable for the opposite. By convention we take $X_1 = 1$ and $Y_1 = 0$. Let $X = \sum_{i = 1}^m X_i$ and $Y = \sum_{i= 1}^m Y_i$. The key observation is that if $Z$ denotes the number of elements in $\{1, \dots, n\}$ that are selected exactly once, then $Z \geq X- Y$.

To see this, observe that if we maintain $Z$ as a set while observing $b_1, b_2, \dots, b_m$, then it follows that whenever $X_i = 1$, we append an element to $Z$ (as its corresponding number $b_i$ will have occurred for the first time and thus be chosen exactly once), and we remove an element from $Z$ only when $Y_i = 1$, as a repeat of a number necessarily implies $Y_i = 1$. It follows that to bound $Z$, it suffices to bound $X- Y$. 

To this end, observe that for any $1 \leq i \leq m$, \textit{regardless} of the outcomes of $X_1, X_2, \dots, X_{i-1}$, $\mathbb{E}[X_i] \geq \frac{n -i + 1}{n}$, as there are at least $n - i + 1$ numbers in $\{1, \dots, n\}$ that have not been chosen yet. It follows that if $X_i^* = \sum_{j = 1}^i X_i - \frac{n -i + 1}{n}$ for $1 \leq i \leq m$, then $X_i^*$ is a sub-martingale (as each term in the sum has expected value at least $0$) satisfying $|X_i^* - X_{i-1}^*| \leq 1$. Applying Azuma's inequality, we see that $$\Pr[X_m^* \geq -\frac{n}{32}] \geq 1 - \exp \left( \frac{-n^2}{2048m}\right) \geq 1 - \exp \left(\frac{-n}{2048}\right).$$

We now apply a similar trick for $Y_1, \dots, Y_m$. In this case, observe that for $1 \leq i \leq m$, \textit{regardless} of the outcomes of $Y_1, \dots, Y_{i-1}$, $\mathbb{E}[Y_i] \leq \frac{i-1}{m}$, as there can be at most $i-1$ numbers that have already been chosen and $Y_i = 1$ if and only if the corresponding $b_i$ is equal to one of those $i-1$ numbers. It follows that $Y_i^* = \sum_{j=1}^i Y_i - \frac{i-1}{m}$ is a super-martingale (as each term has expected value at most $0$) with $|Y_i^* - Y_{i-1}^*| \leq 1$. Applying Azuma's inequality, we see that $$\Pr[Y_m^* \leq \frac{n}{32}] \geq 1 - \exp \left( \frac{-n^2}{2048m}\right) \geq 1 - \exp \left(\frac{-n}{2048}\right).$$ Applying a union bound, we see that with probability at least $1 - 2\exp \left(\frac{-n}{2048}\right)$, $X_m^* \geq \frac{-n}{32}$ and $Y_m^* \leq \frac{-n}{32}$. By substituting these inequalities in, it follows that with probability $1 - 2\exp \left(\frac{-n}{2048}\right)$, $Z$ satisfies
\begin{equation*}
\begin{split}
Z &\geq X - Y \\
&= \sum_{i=1}^m X_i - \sum_{j = 1}^m Y_i \\
&= \sum_{i=1}^m \left(X_i - \frac{n - i+1}{n} \right) + \sum_{i=1}^m \left(\frac{n-i+1}{n}\right) - \sum_{i=1}^m \left(Y_i - \frac{i-1}{n} \right) - \sum_{i=1}^m \left(\frac{i-1}{n}\right) \\
&= X_m^* - Y_m^* + \sum_{i=1}^m \left(\frac{n-i+1}{n}\right) -  \sum_{i=1}^m \left(\frac{i-1}{n}\right) \\
&\geq -\frac{n}{32} - \frac{n}{32} + \sum_{i=1}^m \left(\frac{n-i+1}{n}\right) -  \sum_{i=1}^m \left(\frac{i-1}{n}\right) \\
&= -\frac{n}{16} + m\left(\frac{n + (n-m + 1)}{2n}\right) - \frac{m(m-1)}{2n} \\
&= -\frac{n}{16} + \frac{m}{2n}\left(2n - m + 1 - m + 1\right) \\
&= -\frac{n}{16} + \frac{m(n - m + 1)}{n} \\
&\geq -\frac{n}{16} + \frac{3n}{16} = \frac{n}{8},
\end{split}
\end{equation*}
with the last inequality holding since $\frac{n}{4} \leq m \leq \frac{3n}{4}$. This concludes our proof since we have shown $Z \geq \frac{n}{8}$ with the desired probability.

\end{proof}

We now apply Lemma \ref{lem:combinatorics} to bound the probability that $S \sim p_T^\a$ covers $T$. 

\begin{lemma}\label{lem:with_high_probability}
Let $T \subset [2\a]$ be a set of $\a$ indices, and let $S \sim p_T^\a$ be a set of $\a$ i.i.d points . Then with probability at least $1 - 4\exp\left(-\frac{\a}{2048}\right)$, $S$ covers $T$.
\end{lemma}

\begin{proof}
Let $S = (x_1, x_2, \dots, x_\a)$, and let $A = (a_1, a_2, \dots, a_\a)$ be the unique indices such that $x_i \in a_i$. By Definition \ref{defn:covers}, $L$ and $L'$ are the number of values in $T$ and $[2\a] \setminus T$ that appear exactly once in $A$. We desire to bound the probability that $|L| \geq \frac{\a}{8}$ and $|L'| \geq \frac{\a}{8}$. To do so, the key idea is to condition on $M$, which we define as the number of $1 \leq i \leq \a$ such that $a_i \in T$. 

Suppose that $M = m$. Observe that the conditional distribution of $A$ (viewed as a multiset) given $M = m$ is precisely the distribution obtained by selecting $m$ indices at uniform from $T$ and $\a-m$ indices at uniform from $[2\a] \setminus m$. This holds because $p_T$ is uniform when restricted to $\cup_{i \in T} C_i$ or $\cup_{i \in [2\a] \setminus T} C_i$. Suppose that $\frac{\a}{4} \leq m \leq \frac{3\a}{4}$. Then the same must hold for $\a - m$. it follows by applying Lemma \ref{lem:combinatorics} to selecting $m$ indices from $T$ and $\a-m$ indices from $[2\a]\setminus T$ that with probability  at least $1 - 2\exp\left(-\frac{\a}{2048}\right)$ that $|L| \geq \frac{\a}{8}$ and $|L|' \geq \frac{\a}{8}$. Thus, by summing over all such $m$, we see that 
\begin{equation*}
\begin{split}
\Pr_{S \sim p_T^\a}[|L| \geq \frac{\a}{8}, |L'| \geq \frac{\a}{8}] &= \sum_{m = 1}^\a \Pr_{S \sim p_T^\a}(M = m) \Pr[|L| \geq \frac{\a}{8}, |L'| \geq \frac{\a}{8} | M = m] \\
&\geq \sum_{m = \a/4}^{3\a/4}  \Pr_{S \sim p_T^\a}(M = m)\Pr[|L| \geq \frac{\a}{8}, |L'| \geq \frac{\a}{8} | M = m] \\
&\geq \sum_{m = \a/4}^{3\a/4}\Pr_{S \sim p_T^\a}(M = m)\left(1 - 2\exp\left(-\frac{\a}{2048}\right)\right) \\
&= \left(1 - 2\exp\left(-\frac{\a}{2048}\right)\right) \Pr_{S \sim p_T^\a}[\frac{\a}{4} \leq M \leq \frac{3\a}{4}].
\end{split}
\end{equation*}

To bound the latter probability, we simply apply a Chernoff bound, as $M = \sum_{i=1}^\a \ind(a_i \in T)$ is the sum of $\a$ independent indicator variables each with expected value $\frac{1}{3}$. Using a two sided Chernoff bound, we see that $\Pr[\frac{\a}{4} \leq M \leq \frac{3\a}{4}] \geq 1 - 2\exp \left( -\frac{\a}{144}\right)$. Substituting this, it follows that $$\Pr_{S \sim p_t^\a}[|L| \geq \frac{\a}{8}, |L'| \geq \frac{\a}{8}] \geq \left(1 - 2\exp\left(-\frac{\a}{2048}\right)\right)\left(1 - 2\exp \left( -\frac{\a}{144}\right) \right) \geq 1 - 4\exp\left(-\frac{\a}{2048}\right).$$
\end{proof}

\textbf{Step 3: Constructing $\mathcal{F}$ and $\mathcal{F'}$}

We start by defining $\mathcal{F}$ and $\mathcal{F'}$ as distributions of pairs $(p, A)$ where $p$ is a data distribution and $A$ is a generative algorithm. 

\begin{definition}\label{defn:f}
$\mathcal{F}$ and $\mathcal{F'}$ are the uniform distributions over $\{(p_T, A_T): T \subset [2\a], |T| = \a\}$ and $\{(p_T, A_T'): T \subset [2\a], |T| = \a\}$ respectively. 
\end{definition}

Next, we use $\mathcal{F}$ and $\mathcal{F'}$ to construct distributions $Q$ and $Q'$ over pairs $(S, q)$, where $S$ is a set of points, and $q$ is generated distribution. 

\begin{definition}\label{defn:q}
Let $Q$ be the distribution of $(S, q)$ where $(p_T, A_T) \sim \mathcal{F}$, $S \sim p_T^\a$, and $q \sim A_T(S)$. Similarly, let $Q'$ be the distribution of $(S, q)$ where $(p_T, A_T') \sim \mathcal{F'}$, $S \sim p_T^\a$, and $q \sim A_T'(S)$. 
\end{definition}

Our goal will be to show that $Q$ and $Q'$ follow similar distributions. Our strategy will be to show that for the majority of $(S, q)$ in their supports, they have similar probability masses. To this end, we first characterize the values of $(S, q)$ that we are interested in considering. 

\begin{definition}\label{defn:nice_s_q}
We say that $(S, q)$ is nice if $S$ is a sample of points from some $p_T$, and $q$ is a generated distribution from either $A_T$ or $A_T'$ that has no support over $C_0$. More precisely, $(S,q)$ is nice if the following conditions hold:
\begin{enumerate}
	\item $S \subset \cup_{i = 1}^{2\a} C_i$, with $|S| = \a$.
	\item There exists a set of  $\frac{\a}{8}$ distinct indices, $L_* \subset [2\a]$, such that for $0 \leq i \leq 2\a$, $$q(C_i) = \begin{cases} \frac{\lambda(1+\epsilon)}{3\a} & i \in L_* \\ 0 & i \in [2\a] \setminus L_* \\ 1 - \frac{\lambda (1 + \epsilon)}{24}  & i = 0 \end{cases}$$
	\item For every $i \in L_*$, $|S \cap C_i| = 1$, meaning exactly one element from $S$ is in $C_i$. 
\end{enumerate}
\end{definition}

We now prove a quick lemma relating nice pairs to instances in which $S$ covers $T$.

\begin{lemma}\label{lem:nice_vs_cover}
Let $T \subset [2\a]$ satisfy $|T| = \a$. Let $S \sim p_T^\a$ and let $q$ and $q'$ be generated distributions with $q = A_T(S)$ and $q' = A_T'(S)$. Then the following three are equivalent:
\begin{enumerate}
	\item $(S, q)$ is nice.
	\item $(S, q')$ is nice.
	\item $S$ covers $T$.
\end{enumerate}
\end{lemma} 

\begin{proof}
Suppose $S$ covers $T$. Then the sets $L$ and $L'$ (Definition \ref{defn:covers}) each have size at least $\frac{\a}{8}$ implying that when running $A_T$ or $A_T'$, the set $L_*$ will be non-trivial. This in turn will imply that $(q, S)$ and $(q', S)$ are nice, regardless of the choice of $L_*$.

Otherwise, suppose $S$ does not cover $T$. Then by Definition \ref{defn:a_t}, $A_T(S)$ and $A_{T'}(S)$ will both be the uniform distribution over $C_0$ thus violating Definition \ref{defn:nice_s_q}. 
\end{proof}

We now show that $Q$ and $Q'$ assign identical probability masses to nice pairs. 

\begin{lemma}\label{lem:probability_is_same}
Let $(S, q)$ be a nice pair. Then $Q(S, q) = Q'(S, q)$ with these expressions denoting the probability that $(S,q)$ is chosen over $Q$ and $Q'$ respectively. 
\end{lemma}

\begin{proof}
Let $S = \{x_1, x_2, \dots, x_\a\}$. Let $M$ denote the set of indices in $\{1, 2 \dots, 2\a\}$ such that exactly one point of $S$ lies in the corresponding circle. That is, $M = \{i: |S \cap C_i| = 1\}$. Let $L_*$ be the set of indices in $\{1, 2, \dots 2\a\}$ where $q$ assigns non-trivial probability mass to the corresponding circle. That is, $L_* = \{i: q(C_i) > 0, 1 \leq i\leq 2\a\}$. Since $(S, q)$ is a nice pair (Definition \ref{defn:nice_s_q}), $L_*$ is a subset of $M$, and satisfies $|L_*| = \frac{\a}{8}$. Furthermore, $q$ is uniquely determined by $L_*$.

We now compute $Q(S, q)$ and $Q'(S, q)$ by summing the conditional probabilities of $(S,q)$ given $(p_T, A_T)$ and $(p_T, A_T')$ respectively as $T$ ranges over all subsets. By utilizing the fact that $(S, q)$ is nice (meaning it can only occur if $S$ covers $T$) along with the definition of $A_T$, we have that
\begin{equation*}
\begin{split}
Q(S, q) &= \sum_{|T| = \a: T \subset [2\a]} \frac{1}{\binom{2\a}{\a}}\Pr[(S, q) | p_T, A_T] \\
&= \sum_{|T| = \a: T \subset [2\a]} \frac{1}{\binom{2\a}{\a}} \Pr[S|p_T]\Pr[A_T(S) = q|S, T] \\
&= \sum_{T: S\text{ covers }T} \frac{1}{\binom{2\a}{\a}}\Pr[S|p_T]\Pr[A_T(S) = q|S, T]. \\
&= \sum_{T: S\text{ covers }T} \frac{1}{\binom{2\a}{\a}}\Pr[S|p_T]\frac{\ind\left(L_* \subseteq T\right)}{\binom{|T \cap M|}{\a/8}}.
\end{split}
\end{equation*} 
with the last equality holding because $A_T(S)$ randomly chooses a $\a/8$ element subset of $T \cap M$ for the support of $q$ (see Definition \ref{defn:a_t}). The term $\ind(L_* \subseteq T)$ is necessary because if $L_* \not \subseteq T$, then it is impossible for it to be chosen making the probability $0$. 

Similarly, letting $T^c$ denote the complement of $T$, we have
\begin{equation*}
Q'(S, q) = \sum_{T: S\text{ covers }T} \frac{1}{\binom{2\a}{\a}}\Pr[S|p_T]\frac{\ind\left(L_* \subseteq T^c \right)}{\binom{|T^c \cap M|}{\a/8}},
\end{equation*}
with the only real difference being the support is chosen from $T^c \cap M$ rather than $T \cap M$. 

To show that these sums are equal, we will further group the sums by using $M$ to define an equivalence relation over $\{T: T \subset [2\a], |T| = \a\}$. For $T_1, T_2 \subset [2\a]$, we say they are equivalent if their intersections with $[2\a] \setminus M$, the complement of $M$, are equal. That is, $$T_1 \sim T_2 \Longleftrightarrow T_1 \cap ([2\a] \setminus M) = T_2 \cap ([2\a] \setminus M).$$ 

The usefulness of this equivalence relation is in the following claim.

\textbf{Claim:} Let $T_1 \sim T_2$ be equivalent subsets of $\a$ indices.  Then the following hold:
\begin{enumerate}
	\item $\Pr[S|p_{T_1}] = \Pr[S|p_{T_2}]$.
	\item $|T_1 \cap M| = |T_2 \cap M|$ and $T_1^c \cap M| = |T_2^c \cap M|$.
	\item $S$ covers $T_1$ if and only if $S$ covers $T_2$. 
\end{enumerate}

\begin{proof}
(Of Claim) Let $T$ be any set of indices, let $S = \{x_1, x_2, \dots, x_\a\}$, and let $a_1, a_2, \dots a_\a$ denote the respective indices of the circles that $x_1, \dots, x_\a$ are on. Without loss of generality (relabeling if necessary), suppose that $a_1, a_2, \dots, a_m$ are the unique indices that constitute $M$ (defined above).

Since $p_T$ has probability mass $\frac{1}{3\a}$ on every index in $T$ and $\frac{2}{3\a}$ on the others, we have that the probability density of $S$ (denoted $\Pr[S|p_{T_1}]$) satisfies,
\begin{equation*}
\begin{split}
\Pr[S|p_{T}] &= \prod_{i= 1}^\a \frac{2 - \ind(a_i \in T)}{\a}\frac{1}{2\pi}\\
&= \left(\prod_{i=1}^m \frac{2 - \ind(a_i \in T \cap M)}{\a}\frac{1}{2\pi}\right)\left(\prod_{i=m+1}^{\a} \frac{2 - \ind(a_i \in T \cap ([2\a] \setminus M))}{\a}\frac{1}{2\pi}\right) \\
&= \left(\frac{2^{|T^c \cap M|}}{(2\pi \a)^m}\right) \left(\prod_{i=m+1}^{\a} \frac{2 - \ind(a_i \in T \cap ([2\a] \setminus M))}{\a}\frac{1}{2\pi}\right),
\end{split}
\end{equation*}
with the last equality exploiting the fact that $\{a_1, a_2, \dots, a_m\}$ precisely equals $M$ (by the definition of $M$). Next, observe that if $T_1 \sim T_2$, then by definition, $T_1 \cap [2\a] \setminus M = T_2 \cap [2\a \setminus M$ implying that the second part of the product is equal. However, since $|T_1| = |T_2| = \a$, the first part must be equal as well, as $|T^c \cap M| = \a - |T^c \cap [2\a] \setminus M|$. It follows that the probability densities are the same. Note that this observation also implies the second claim, that $|T_1 \cap M| = |T_2 \cap M|$ and $T_1^c \cap M| = |T_2^c \cap M|$

Finally, to show the second part of the claim, we simply observe that for a set $T$, the sets $L$ and $L'$ from Definition \ref{defn:covers} are precisely $T \cap M$ And $T^c \cap M$. For $T = T_1, T_2$, by the second claim, these have equal sizes.
\end{proof}

We now return to the proof of Lemma \ref{lem:probability_is_same}. Having shown the claim, we now return to our original computation. Let $T_1, T_2, \dots, T_r$ denote sets of $\a$ indices with $[T_1], [T_2], \dots, [T_r]$ denoting their respective equivalence classes such that $[T_1], \dots, [T_r]$ partition $\{T: S\text{ covers }T\}$. This is possible from the third part of our claim.

For $1 \leq i \leq r$, let $m_i = |T_i \cap M|$ and $m_i' = |T_i^c \cap M|$ where $T_i^c$ denotes the complement of $T_i$. It follows from second part of our claim that $|T \cap M|, |T^c \cap M|$ both equal $m_i$ as well for all $T \in [T_i]$. 

By partitioning our sum for $Q(S, q)$ in using $[T_1], \dots, [T_r]$, we have
\begin{equation*}
\begin{split}
Q(S, q) &= \sum_{T: S\text{ covers }T} \frac{1}{\binom{2\a}{\a}}\Pr[S|p_T]\frac{\ind\left(L_* \subseteq T\right)}{\binom{|T \cap M|}{\a/8}} \\
&= \sum_{i = 1}^r \sum_{T \in [T_i]} \frac{1}{\binom{2\a}{\a}}\Pr[S|p_T]\frac{\ind\left(L_* \subseteq T\right)}{\binom{|T \cap M|}{\a/8}} \\
&= \sum_{i = 1}^r \frac{\Pr[S|p_{T_i}]}{\binom{2\a}{\a}}\sum_{T \in [T_i]} \frac{\ind\left(L_* \subseteq T\right)}{\binom{m_i}{\a/8}}  \\
&= \sum_{i = 1}^r \frac{\Pr[S|p_{T_i}]}{\binom{2\a}{\a}}\frac{\binom{m - \a/8}{m_i - \a/8}}{\binom{m_i}{\a/8}},
\end{split}
\end{equation*}
with the last equality coming by counting the number of $T \in [T_i]$ such that $L_* \subseteq T$. This counting problem essentially forces all $\a/8$ elements in $L_*$ to be in $T$ leaving us to choose the remaining elements in $M$ that can be part of $T$.

By using the exact same line of reasoning for $Q'(S,q)$, we have 
\begin{equation*}
\begin{split}
Q'(S, q) &= \sum_{T: S\text{ covers }T} \frac{1}{\binom{2\a}{\a}}\Pr[S|p_T]\frac{\ind\left(L_* \subseteq T^c \right)}{\binom{|T^c \cap M|}{\a/8}} \\
&= \sum_{i = 1}^r \sum_{T \in [T_i]} \frac{1}{\binom{2\a}{\a}}\Pr[S|p_T]\frac{\ind\left(L_* \subseteq T^c\right)}{\binom{|T^c \cap M|}{\a/8}} \\
&= \sum_{i = 1}^r \frac{\Pr[S|p_{T_i}]}{\binom{2\a}{\a}}\sum_{T \in [T_i]} \frac{\ind\left(L_* \subseteq T^c\right)}{\binom{m_i'}{\a/8}}  \\
&= \sum_{i = 1}^r \frac{\Pr[S|p_{T_i}]}{\binom{2\a}{\a}}\frac{\binom{m - \a/8}{m_i' - \a/8}}{\binom{m_i'}{\a/8}},
\end{split}
\end{equation*}
Here the only difference ends up being that we use $m_i'$ instead of $m_i$ since we have effectively replaced $T$ with $T^c$. However, this replacement only takes place for $q$, the component of the probability that deals with $S$ is identical for both $Q$ and $Q'$. 

Finally, based on these equations, it suffices to show that $\frac{\binom{m - \a/8}{m_i' - \a/8}}{\binom{m_i'}{\a/8}} = \frac{\binom{m - \a/8}{m_i - \a/8}}{\binom{m_i}{\a/8}}$. To do so, since $m_i = |T_i \cap M|$ and $m_i' = |T_i^c \cap M|$, it follows that $m_i + m_i' = m$. Using this, we have that
\begin{equation*}
\begin{split}
\frac{\binom{m - \a/8}{m_i - \a/8}}{\binom{m_i}{\a/8}} &= \frac{\left(m - \a/8 \right)! \left(\a/8\right)!(m_i - \a/8)!}{(m_i  - \a/8)!(m - m_i)!m_i!} \\
&= \frac{(m - \a/8)!(\a/8)!}{m_i'!m_i!}.
\end{split}
\end{equation*}
Applying the same manipulation to $\frac{\binom{m - \a/8}{m_i' - \a/8}}{\binom{m_i'}{\a/8}}$ completes the proof.

\end{proof}

\textbf{Step 4: finishing the overall proof.}

Let $\a$ be a sufficiently large integer. It suffices to show that there exists a probability distribution $p$ with $\frac{1}{9\a} \leq p_\epsilon \leq \frac{2}{3\a}$ such that $m_p(\epsilon, \delta) > \a$. Assume towards a contradiction that no such $p$ exists, meaning that $m_p(\epsilon, \delta) \leq \a$ for all $p$ satisfying the above.  

Let $T \subset [2\a]$ satisfy $T = [2\a]$. By Lemma \ref{lem:p_T_is_regular}, $\frac{1}{9\a} \leq (p_T)_\epsilon \leq \frac{2}{3\a}$. It follows that with probability at least $1- \delta$ over $S \sim p_T^\a$ and  $q \sim A_T(S)$ along with the randomness of $B$, $$\c_q^{-\epsilon} - \epsilon \leq B(S, q) \leq \c_q^{\epsilon} + \epsilon,$$ with $\c_q^{-\epsilon}, \c_q^{\epsilon}$ denoting the appropriate data-copying rates for $q$ with respect to $p$.  

By Lemma \ref{lem:data_copy_bound}, if $S$ covers $T$, then $\c_q^{-\epsilon} \geq \frac{\lambda(1+\epsilon)}{24} = \frac{13\frac{4}{3}}{24} > \frac{2}{3}.$ By Lemma \ref{lem:with_high_probability}, $S$ covers $T$ with probability at least $1 - 4\exp\left(-\frac{\a}{2048}\right).$ Substituting this, we have 
\begin{equation*}
\begin{split}
1 - \delta &\leq \E_{S \sim p_T^\a} \E_{q \sim A_T(S)} \E_B \ind\left(B(S, q) \geq \c_q^{-\epsilon}- \epsilon \right)\\
&= \E_{S \sim p_T^\a} \ind\left(S\text{ does not cover }T\right) + \E_{S \sim p_T^\a}\ind\left(S\text{ covers }T\right)\E_{q \sim A_T(S)}\E_B \ind\left(B(S, q) > \frac{1}{3} \right) \\
&\leq 4\exp\left(-\frac{\a}{2048}\right) + \E_{S \sim p_T^\a}\ind\left(S\text{ covers }T\right)\E_{q \sim A_T(S)}\E_B \ind\left(B(S, q) > \frac{1}{3} \right),
\end{split}
\end{equation*}

with the substitutions for $\c_q^{-\epsilon} - \epsilon$ utilizing that $\epsilon = \frac{1}{3}$. 

Applying this over the distribution, $\mathcal{F}$ (Definition \ref{defn:f}), which comprises of all $(p_T, A_T)$ with $T$ chosen at uniform over all subsets of size $\a$, and then substituting the definition of $Q$ (Definition \ref{defn:q}), we have 
\begin{equation}\label{eqn:final_good}
\begin{split}
1 - \delta -  4\exp\left(-\frac{\a}{2048}\right) &\leq \mathbb{E}_{(p_T, A_T) \sim \mathcal{F}} \mathbb{E}_{S \sim p_T^\a}\mathbb{E}_{q \sim A_T(S)} \ind\left(S\text{ covers }T\right)\mathbb{E}_{B}\ind\left(B(S, q) > \frac{1}{3}\right)\\
&= \mathbb{E}_{(S, q) \sim Q}\ind\left((S, q)\text{ is nice}\right)\mathbb{E}_{B}\ind\left(B(S, q) > \frac{1}{3}\right) \\
&= \frac{1}{1 - \Pr_{(S, q) \sim Q}[(S, q)\text{ is not nice}]}\mathbb{E}_{(S, q) \sim Q_*}\mathbb{E}_{B}\ind\left(B(S, q) > \frac{1}{3}\right) \\
&\leq \mathbb{E}_{(S, q) \sim Q_*}\mathbb{E}_{B}\ind\left(B(S, q) > \frac{1}{3}\right),
\end{split}
\end{equation}
where $Q_*$ denotes the marginal distribution of $Q$ over all nice (Definition \ref{defn:nice_s_q}) pairs $(S, q)$. Note that the manipulation above holds because of Lemma \ref{lem:nice_vs_cover}, which implies that $(S, q)$ is nice if and only if $S$ covers $T$. 

Next, we apply the same exact reasoning to the pair $(p_T, A_T')$. To this end, we have that with probability at least $1 - \delta$ over $S \sim p_T^\a$, $q \sim A_T'(S)$, along with the randomness of $B$, $$\c_q^{-\epsilon} - \epsilon \leq D(S, q) \leq \c_q^{\epsilon} + \epsilon.$$ 

By Lemma \ref{lem:data_copy_bound}, if $S$ covers $T$, then $\c_q^\epsilon = 0$. Applying the same argument as above using Lemma \ref{lem:with_high_probability}, we have that 
\begin{equation*}
\begin{split}
1 - \delta -  4\exp\left(-\frac{\a}{2048}\right)\leq \E_{S \sim p_T^\a}\ind\left(S\text{ covers }T\right)\E_{q \sim A_T'(S)}\E_B \ind\left(B(S, q) \leq \frac{1}{3} \right).
\end{split}
\end{equation*}
Applying this over the distribution $\mathcal{F}'$ (Definition \ref{defn:f}) and using a similar set of manipulations as we did with $\mathcal{F}$ and $Q$, we have that
\begin{equation}\label{eqn:final_prime}
\begin{split}
1 - \delta -  4\exp\left(-\frac{\a}{2048}\right) &\leq \mathbb{E}_{(p_T, A_T') \sim \mathcal{F'}} \mathbb{E}_{S \sim p_T^\a}\mathbb{E}_{q \sim A_T'(S)} \ind\left(S\text{ covers }T\right)\mathbb{E}_{B}\ind\left(B(S, q) \leq \frac{1}{3}\right)\\
&\leq \mathbb{E}_{(S, q) \sim Q_*'}\mathbb{E}_{B}\ind\left(B(S, q) \leq \frac{1}{3}\right),
\end{split}
\end{equation}
where $Q_*'$ denotes the marginal distribution of $Q'$ over nice pairs $(S, q)$. 

Finally, by Lemma \ref{lem:probability_is_same}, $Q_*'$ and $Q_*$ follow the exact same distribution. This means that summing equations \ref{eqn:final_good} and \ref{eqn:final_prime}, we can combine the summands \textit{inside} the expectation giving us that $$2 - 2\delta - 8\exp\left(-\frac{\a}{2048}\right) \leq \E_{(S, q)\ sim Q_*} \E_B \left(\ind\left(B(S, q) > \frac{1}{3}\right)+\ind\left(B(S, q) \leq \frac{1}{3}\right) \right) = 1.$$ This gives a contradiction as this equation is clearly false when $\a$ is sufficiently large (as $\delta = \frac{1}{3}$).

\end{proof}

\end{document}